\documentclass{article}
\usepackage{microtype}
\usepackage{graphicx}
\usepackage{subfigure}
\usepackage{booktabs} % for professional tables
\usepackage{amsfonts}       % blackboard math symbols
\usepackage{nicefrac}       % compact symbols for 1/2, etc.
\usepackage{microtype}      % microtypography

\usepackage{amsmath}
\usepackage{amssymb}
\usepackage{amsthm}
\usepackage{amssymb,mathtools,amsfonts}

\usepackage{hyperref}

\usepackage[accepted]{icml2021}

\usepackage{verbatim}

\newcommand{\inner}[3]{\left\langle #1,#2\right\rangle_{#3}}

\newcommand{\athroisma}[4]{\sum_{#1=#2}^{#3}#4}

\newcommand{\ar}{\mbf{R}}

\newenvironment{proofof}[1]{\begin{proof}[{#1}]}{\end{proof}}

\usepackage{xcolor}
\definecolor{ethblue}{cmyk}{1.0,0.7,0,0.3}
\definecolor{tangored}{HTML}{A40000}
\definecolor{tangogreen}{HTML}{4E9A06}

\usepackage{tikz}
\usetikzlibrary{arrows, decorations.markings}\usetikzlibrary{positioning}
\usetikzlibrary{shapes.multipart}

\tikzstyle{block} = [draw, line width=1.0, black, fill=white, minimum size=1.0cm, rounded corners=0.1cm]

\tikzstyle{vecArrow} = [
  thick,
  decoration={
    markings,
    mark=at position 1 with {\arrow[semithick,scale=1.25]{open triangle 60}}
  },
  double distance=2pt,
  shorten >= 6.8pt,
  preaction = {decorate},
  postaction = {
    draw,
    line width=2pt,
    white,
    shorten >= 4.5pt
  }
]
\tikzstyle{vecArrowDouble} = [
  thick,
  decoration={
    markings,
    mark=at position 0 with {\arrowreversed[semithick,scale=1.25]{open triangle 60}},
    mark=at position 1 with {\arrow[semithick,scale=1.25]{open triangle 60}}
  },
  double distance=2pt,
  shorten >= 6.8pt,
  shorten <= 6.8pt,
  preaction = {decorate},
  postaction = {
    draw,
    line width=2pt,
    white,
    shorten >= 4.5pt,
    shorten <= 4.5pt
  }
]

\usepackage{pgfplots}
\usepgfplotslibrary{fillbetween}

\newcommand{\expert}{{\pi_{\textup{E}}}}
\newcommand{\apprentice}{\pi_{\textup{A}}}
\newcommand{\mbs}{\boldsymbol}
\newcommand{\cost}{\mbf{c}}
\newcommand{\true}{\mbf{c_{\textup{true}}}}
\newcommand{\weight}{\mbf{w}}
\newcommand{\uv}{\mbf{u}}
\newcommand{\mv}{\mbs{\mu}}
\newcommand{\op}{\mbf{T}_{\gamma}}
\newcommand{\initial}{\mbs{\nu}_0}
\newcommand{\val}{\mbf{V}}
\newcommand{\RL}[1]{(\textbf{\textup{RL}}_{#1})}
\newcommand{\LfD}{(\textbf{\textup{LfD}}_{\pi_{\textup{E}}})}
\newcommand{\primal}{(\textbf{\textup{P}}_{\pi_{\textup{E}}})}
\newcommand{\dual}{(\textbf{\textup{D}}_{\pi_{\textup{E}}})}
\newcommand{\thv}{\mbs{\theta}}
\newcommand{\lv}{\mbs{\lambda}}
\newcommand{\lag}{\mcf{L}}
\newcommand{\rlag}{{\mcf{L}_r}}
\newcommand{\mlag}{\bar{\mcf{L}}_r}
\newcommand{\elag}{\widehat{\mcf{L}}_r}
\newcommand{\phim}{\mbs{\Phi}}
\newcommand{\psim}{\mbs{\Psi}}
\newcommand{\cma}{\mbf{C}}
\newcommand{\psiv}{\mbs{\psi}}
\newcommand{\phiv}{\mbs{\phi}}

\newcommand{\fev}{\mbs{\rho}_{\cma}(\expert)}
\newcommand{\efev}{\mbs{\rho}_{\cma}(\widehat{\expert})}
\newcommand{\mth}{\mv_{\thv}}
\newcommand{\ul}{\uv_{\lv}}
\newcommand{\cw}{\cost_{\weight}}
\newcommand{\cbg}{C_{\beta,\gamma}}
\newcommand{\mexp}{\mv_{\expert}}
\newcommand{\zv}{\mbf{z}}
\newcommand{\pth}{\pi_{\thv}}
\newcommand{\pem}{\pi_{\mv}}
\newcommand{\tha}{\thv_{\textup{A}}}
\newcommand{\la}{\lv_{\textup{A}}}
\newcommand{\wa}{\weight_{\textup{A}}}

\newcommand{\ut}{\tilde{\uv}}

\newcommand{\lt}{\tilde{\lv}}
\newcommand{\ma}{\mv_{\textup{A}}}
\newcommand{\ua}{\uv_{\textup{A}}}
\newcommand{\mmat}{\mbf{M}}
\newcommand{\bmat}{\mbf{B}}
\newcommand{\pmat}{\mbf{P}}
\newcommand{\eith}{\mbf{e}_{i_{\thv}}}
\newcommand{\cith}{\cost_{i_{\thv}}}

\newcommand{\mpa}{\mv_{\apprentice}}
\newcommand{\weirdv}{\val_{\cith}^{\pth}}
\newcommand{\weirdrho}{\rho_{\cith}(\pth)}

 \newcommand{\thN}{\widehat{\thv}_N}
 \newcommand{\lN}{\widehat{\lv}_N}
 \newcommand{\wN}{\widehat{\weight}_N}
 \newcommand{\pN}{\widehat{\pi}_N}
 \newcommand{\zN}{\widehat{\zv}_N}
 
 \newcommand{\fn}{\mathcal{F}_{n-1}}
 
 \newcommand{\hi}{\hat{\iota}}
 \newcommand{\hx}{\hat{x}}
 \newcommand{\ha}{\hat{a}}
 \newcommand{\hy}{\hat{y}}
 \newcommand{\ti}{\tilde{\iota}}
 \newcommand{\tx}{\tilde{x}}
 \newcommand{\ta}{\tilde{a}}
 \newcommand{\ty}{\tilde{y}}
 \newcommand{\txp}{\tilde{x}'}
 \newcommand{\bi}{\bar{\iota}}
 \newcommand{\bx}{\bar{x}}
 \newcommand{\ba}{\bar{a}}
 \newcommand{\gv}{\mbf{g}}
 \newcommand{\gth}{\mbf{g}_{\thv}}
 \newcommand{\gl}{\mbf{g}_{\lv}}
 \newcommand{\gw}{\mbf{g}_{\weight}}
 \newcommand{\sadav}{\epsilon_{\textup{sad}}(\zN)}
 
 \newcommand{\mpm}{\mv_{\pem}}
 \newcommand{\sad}{\epsilon_{\textup{sad}}(\boldsymbol{\theta},\boldsymbol{\lambda},\weight)}
 \newcommand{\lvz}{\lv_0}
 \newcommand{\ulz}{\uv_{\lvz}}
 \newcommand{\wtrue}{\weight_{\textup{true}}}

\newcommand{\iid}{i.i.d.\ }

\newcommand{\mcf}{\mathcal}
\newcommand{\mbf}{\mathbf}

\DeclareMathOperator{\Var}{\mbf{Var}}
\DeclareMathOperator{\Exp}{\mbf{E}}
\DeclareMathOperator{\Prob}{\mbf{P}}
\newcommand{\ones}{\mbf{1}}
\renewcommand{\Re}{\mbf{R}}

\newcommand{\sspace}{\mathcal{X}}     % state space
\newcommand{\aspace}{\mathcal{A}}     % action space
\newcommand{\abs}[1]{|#1|}
\newcommand{\norm}[1]{\lVert#1\rVert}

\newcommand{\innerprod}[2]{\left\langle{#1},{#2}\right\rangle}

\newtheorem{definition}{Definition}
\newtheorem{lemma}{Lemma}
\newtheorem{proposition}{Proposition}

\newtheorem{theorem}{Theorem}
\newtheorem{corollary}{Corollary}

\newtheorem{assumption}{Assumption}

\usepackage{xcolor}

	\icmltitlerunning{Efficient Performance Bounds for Primal-Dual Reinforcement Learning from Demonstrations}
	
\begin{document}
		
		\twocolumn[
		\icmltitle{Efficient Performance Bounds for Primal-Dual Reinforcement\\ Learning from Demonstrations}
		
		
		
		\icmlsetsymbol{equal}{*}
		
		\begin{icmlauthorlist}
			\icmlauthor{Angeliki Kamoutsi}{to}
			\icmlauthor{Goran Banjac}{to}
			\icmlauthor{John Lygeros}{to}
			%\icmlauthor{Iaesut Saoeu}{ed}
			%\icmlauthor{Fiuea Rrrr}{to}
			%\icmlauthor{Tateu H.~Yasehe}{ed,to,goo}
			%\icmlauthor{Aaoeu Iasoh}{goo}
			%\icmlauthor{Buiui Eueu}{ed}
			%\icmlauthor{Aeuia Zzzz}{ed}
			%\icmlauthor{Bieea C.~Yyyy}{to,goo}
			%\icmlauthor{Teoau Xxxx}{ed}
			%\icmlauthor{Eee Pppp}{ed}
		\end{icmlauthorlist}
		
		\icmlaffiliation{to}{Automatic Control Laboratory, ETH Zurich, Switzerland}
		%\icmlaffiliation{goo}{Googol ShallowMind, New London, Michigan, USA}
		%\icmlaffiliation{ed}{School of Computation, University of Edenborrow, Edenborrow, United Kingdom}
		
		\icmlcorrespondingauthor{Angeliki Kamoutsi}{kamoutsa@ethz.ch}
		%\icmlcorrespondingauthor{Eee Pppp}{ep@eden.co.uk}
		
		\icmlkeywords{Machine Learning, ICML}
		
		\vskip 0.3in
		]
		
		
		
		\printAffiliationsAndNotice{}  % leave blank if no need to mention equal contribution
		
		\begin{abstract}
			We consider large-scale Markov decision processes with an unknown cost function and address the problem of learning a policy from a finite set of expert demonstrations.
			We assume that the learner is not allowed to interact with the expert and has no access to reinforcement signal of any kind.
			Existing inverse reinforcement learning methods come with strong theoretical guarantees, but are computationally expensive, while state-of-the-art policy optimization algorithms achieve significant empirical success, but are hampered by limited theoretical understanding.
			To bridge the gap between theory and practice, we introduce a novel bilinear saddle-point framework using Lagrangian duality.
			The proposed primal-dual viewpoint allows us to develop a model-free provably efficient algorithm through the lens of stochastic convex optimization. The method enjoys the advantages of simplicity of implementation, low memory requirements, and computational and sample complexities independent of the number of states. We further present an equivalent no-regret online-learning interpretation.
		\end{abstract}
		
		\section{Introduction}
		\label{submission}
		Reinforcement learning (RL) is an area in machine learning with connections to control and optimization that has shown tremendous success in large-scale real-world applications, such as robotics, aritificial intelligence, cognitive autonomy, operations research, and healthcare~\cite{Tesauro:2002,Mnih:2015,Chen:2017,Vamvoudakis:2020}. It studies the problem of learning to act optimally in a sequential decision-making problem, while interacting with an unknown environment~\cite{Bertsekas:1996,Sutton:1998}. 
		
		In the standard RL setting a cost signal is given to instruct agents how to complete the desired task. However, oftentimes encoding preferences using demonstrations provided by an expert, is easier than designing a cost function~\cite{Pomerleau:1991,Russell:1998,Bagnell:2015}. In such cases, the 
		goal of \emph{learning from demonstrations} (LfD) or \emph{imitation learning} (IL)  is to learn a policy that achieves or even surpasses the performance of the expert policy.

		A lot of methods have been proposed to solve the LfD problem. The most straightforward approach is \emph{behavior cloning}, which casts the problem as a supervised learning problem, in which the goal is to learn a map from states to optimal actions~\cite{Pomerleau:1991}. Although behavior cloning is simple and easy to implement, the crucial i.i.d. assumption made in supervised learning is violated. As a result, the approach suffers from the problem of \emph{cascading errors}, which is related to \emph{covariate shift}~\cite{Bagnell:2015,Ho:2016}.  Later works~\cite{ross2011reduction} eliminate this distribution mismatch by formulating the problem as a no-regret online learning problem, but require interaction with the expert. On the contrary, in this paper, we consider the \emph{batch} LfD scenario, i.e., the learner can observe only a finite set of expert demonstrations, is not allowed to query the expert for more data while training, and is not provided any reinforcement signal.
		
		Inverse reinforcement learning (IRL)~\cite{Abbeel:2004} is a prevalent approach to LfD. In this paradigm, the learner first infers the unknown cost function that the expert tries to minimize and then uses it to reproduce the optimal behavior. IRL algorithms do not suffer from the problem of cascading errors because the training takes place over entire expert trajectories, rather than individual actions. In addition, since the recovered cost function ``explains'' the expert behavior, they can easily generalize to unseen states or even new environments. Note however, that most existing IRL algorithms \cite{Abbeel:2004,Ratliff:2006,Syed:2007,Neu:2007,Ziebart:2008,Abbeel:2008,Levine:2010,Levine:2011} are computationally expensive because they use RL as a subroutine.
		
		On the other hand, under the assumption of a linearly parameterized cost class, one can frame the problem as a single convex program~\cite{syed2008apprenticeship}, bypassing the intermediate step of learning the cost function. Although the associated program can be solved exactly for small-sized MDPs, the approach suffers from the \emph{curse of dimensionality}, making it intractable for large-scale problems. 
		
		The formulations and reasoning in~\cite{syed2008apprenticeship} formed the ground and inspired later state-of-the-art policy optimization algorithms~\cite{Ho:2016,Ho:2016b}. In particular, the authors in~\cite{Ho:2016} consider the case of linearly paremeterized cost classes and present policy gradient algorithms, which are parallel to those proposed in \cite{Williams:1992,Schulman:2015} for RL. Moreover, \cite{Ho:2016b} propose a generative adversarial imitation learning (GAIL) method for the case of nonlinear costs. They do so, by formulating the problem as minimax optimization and drawing a connection to generative adversarial networks \cite{Goodfellow:2014}. In particular, GAIL solves IL with alternating updates of both policy and cost functions. These approaches are model-free and achieve significant empirical success in	challenging benchmark tasks. However, in general the associated minimax problem is highly nonconvex-nonconcave and as a result remains hampered by limited theoretical understanding. 
		
		Indeed, the global convergence properties of alternating policy gradient schemes for the minimax formulation of GAIL have been studied only for linear (although infinite-dimensional) Markov decision processes (MDPs) and linear or linearizable costs~\cite{Cai:2019,Zhang:2020}. It however remains unclear whether such neural policy gradient methods converge to the optimal policy or if they converge at all, for the case of nonlinear dynamics. 
		As a result, provably efficient policy optimization schemes for the LfD problem beyond the linear setting remain largely unexplored. 
		
		 \textbf{Contributions.} In an attempt to tackle this longstanding question, in this work we present a convex-analytic viewpoint of the LfD problem.
         To this end, we adopt the apprenticeship learning (AL) formalism which carries the assumption that the unknown true cost function can be
         represented as a weighted combination of some known basis functions, where the true unknown weights specify how different desiderata should be traded-off~\cite{Abbeel:2004,Syed:2007,syed2008apprenticeship,Ho:2016b,Brown:2020a}. In particular, we make no restrictive assumptions on the MDP model (linearity or ergodicity). 
         
         Following the recent line of works~\cite{Chen:2018,Lee:2019a,Wang:2019,Bas-Serrano:2020b,Cheng:2020,Jin:2020,Shariff:2020} on approximate linear programming (ALP) for large-scale MDPs, we formulate the LfD problem as a bilinear saddle-point problem in light of Lagrangian duality. We study primal-dual optimality conditions and prove relations between saddle points of the Lagrangian function and optimal solutions to the LfD problem. In particular, we show that under the expert optimality assumption, the set of solutions to the dual linear program (LP) characterizes the set of solutions to the inverse problem, i.e., the set of cost functions for which the expert is optimal. Moreover, in this case, we show that the expert policy, the true cost function, and the true optimal value function form a saddle point of the proposed Lagrangian.  
		 
		Analogous to ALP, we obtain a linearly-relaxed saddle-point formulation by limiting our search to a linear subspace defined by a small number of features. We exhibit a formal link between approximate saddle points of the reduced Lagrangian and the optimality gap of our problem.  
		
		The aforementioned analysis lays theoretical foundations for a provably efficient stochastic primal-dual algorithm. By using linear function approximators, we propose a mirror-descent-based algorithm with a generative model and derive explicit probabilistic performance bounds on the quality of the extracted policy. A salient feature of the algorithm is that its sample complexity does not depend on the size of the state space but instead on the number of approximation features. We note however that our algorithm degrades with the approximation error due to the linear approximation architecture. We consider  both the case of weak and strong linear features~\cite{Laksh:2018,Shariff:2020}. Finally, we present an equivalent no-regret online-learning interpretation of our primal-dual algorithm. This kind of reduction has been studied up to now only for the online IL setting~\cite{ross2011reduction}, where interaction with the expert is allowed.

		\textbf{Related works} Our work is related to the LP approach to AL~\cite{syed2008apprenticeship} and the theoretical analysis of GAIL made in~\cite{Cai:2019,Chen:2020a,Zhang:2020}. Unlike~\cite{syed2008apprenticeship}, who consider the case of tabular MDPs and known dynamics, we consider large-scale problems and only access to a simulator for the MDP model. We revisit the LP approach in~\cite{syed2008apprenticeship} and focus instead on its saddle-point formulation due to its potential for scalable algorithms with theoretical guarantees~\cite{Cheng:2020,Nachum:2020}. The authors in~\cite{Cai:2019} study the global convergence properties of GAIL for the linear quadratic regulator problem by extending the results in~\cite{Fazel:2018}, while the authors in~\cite{Zhang:2020} consider the case of infinite-dimensional linear MDPs and linearizable cost functions. Similar to~\cite{Wang:2020} the policies and cost functions are approximated by overparameterized ReLU neural networks. In our work, we consider general nonlinear MDPs without any restrictive assumption such as linearity or ergodicity, and employ linear function approximators. The authors in~\cite{Chen:2020a} study the convergence and generalization of GAIL for general MDPs. However they only prove convergence to a stationary point. On the contrary, our algorithm produces a nearly-optimal policy.
		
		Our work builds upon a vast line of works on ALP for forward  RL~\cite{DeFarias:2003,Abbasi-Yadkori:2014,Chen:2018,Laksh:2018,MohajerinEsfahani:2018,Wang:2019,Lee:2019a,Bas-Serrano:2020a,Cheng:2020,Jin:2020,Shariff:2020,Bas-Serrano:2020b}. The LP approach to MDPs, which dates back to~\cite{Manne:1960,Hernandez-Lerma:1996,Borkar:1988}, has recently gained traction as an alternative to dynamic programming techniques, for its advantage to lead to problem formulations that are directly amenable to modern large-scale stochastic optimization methods. Moreover, this optimization-based approach can tackle unconventional problems involving additional safety constraints or secondary costs, where traditional dynamic programming techniques are not applicable \cite{Hernandez-Lerma:2003,Dufour:2013, Shafieepoorfard:2013}. In this paper, we develop an ALP framework for the LfD problem. While the forward policy optimization problem tries to minimize the total expected cost, the LfD problem is a minimax problem that tries to match the expert across a given cost class. In our saddle-point formulation
	    the cost function is not fixed but is itself a decision variable. Thus, the variation of the cost function during the algorithm
	    continuously changes the Lagrangian making the analysis of the problem more challenging. 
		
		%============================================================================================================================================================================
		%New section
		%============================================================================================================================================================================	
		\section{Preliminaries and Problem Setup}
			\subsection{Basic Definitions and Notations}
		We denote by $\ar^n_+$ and $\ar^n_{++}$ the sets of $n$-dimensional vectors with nonnegative and positive real elements, respectively.
		For a matrix $\mbf{A}\in\ar^{m\times n}$, its $p$-norm is defined by $\norm{\mbf{A}}_p\triangleq\sup\{\norm{\mbf{Ax}}_p\mid\norm{\mbf{x}}_p=1\}$.
		%\GB{Do you mean $\mbf{A}\in\ar^{m\times n}$? Otherwise, $\mbf{A}\mbf{x}$ with $\mbf{x}\in\ar^n$ does not work.}
		For vectors $\mbf{x}$ and $\mbf{y}$, we denote by $\inner{\mbf{x}}{\mbf{y}}{}$ the usual inner product. Moreover, $\mbf{x}\le \mbf{y}$ denotes elementwise inequality, i.e., $x_i\le y_i$ for all $i$. We use $\mbf{1}$ and $\mbf{0}$ to denote vectors with all elements equal to one and zero, respectively. The set of probability distributions on a finite set $\mcf{S}$ is denoted by $\Delta_{\mcf{S}}$, i.e., $\Delta_{\mcf{S}}\triangleq\{\mbf{p}\in\ar_+^{|\mcf{S}|}\mid\sum_{s\in\mcf{S}}p(s)=1\}$, where $|\mcf{S}|$ is the cardinality of $\mcf{S}$. Sums spanning over the spaces $\sspace$ and $\aspace$ will be simply denoted by $\sum_x$ and $\sum_a$, respectively. For $\mv,\mv'\in\Delta_{\mathcal{X}}$ their \emph{Kullback–Leibler divergence} is given by $\textup{KL}(\mv||\mv')\triangleq\sum_{x}\mu(x)\log\tfrac{\mu(x)}{\mu'(x)}$. For a nonempty closed convex set $\Theta$, the Euclidean projection of $\mbf{x}$ onto $\Theta$ is given by $\Pi_{\Theta}(\mbf{x})\triangleq\arg\min_{\mbf{y}\in\Theta}\norm{\mbf{x}-\mbf{y}}_2$.
		\subsection{Reinforcement Learning}
		A finite MDP is given by a tuple $\big( \sspace,\aspace,P,\initial,\cost,\gamma\big)$, where $\sspace$ is the state space, $\aspace$ is the action space, $P:\sspace\times\aspace\to \Delta_{\sspace}$ is the transition law, $\initial\in\Delta_{\sspace}$ is the initial state distribution, $\cost\in[-1,1]^{|\sspace||\aspace|}$ is the one-stage cost, and $\gamma\in(0,1)$ is the discount factor. We focus on problems where $\sspace$ and $\aspace$ are too large to be enumerated. 
		
		The MDP models a controlled discrete-time stochastic system with initial state $x_0\sim\initial$.
		At each round $t$, if the system is in state $x_t=x\in\sspace$ and  the action $a_t=a\in\aspace$ is taken, then a cost $c(x,a)$ is incurred, and the system transitions to the next state $x_{t+1}\sim P(\cdot|x,a)$. 
		
		A \emph{stationary Markov policy} is a map $\pi\colon\sspace\to\Delta_{\aspace}$, and $\pi(a|x)$ denotes the probability of choosing action $a$, while being in state $x$. We denote the space of stationary Markov policies by $\Pi_0$.
		
		The \emph{value function} $\val_\cost^\pi\in\ar^{|\sspace|}$ of $\pi$, given a cost $\cost$, is defined by
		$
		V_\cost^{\pi}(x) \triangleq(1-\gamma)\Exp_x^{\pi}\Big[\sum_{t=0}^\infty \gamma^t c(x_t, a_t)\Big],\label{V-function}
		$
		where $\Exp^{\pi}_{x}$ denotes the expectation with respect to the trajectories generated by $\pi$ starting from $x_0=x$.
		%We remark that the definitions above contain a $(1-\gamma)$. We adopt this setup in order to make writing more compact.
		
		The goal of RL is to solve the following optimal control problem
		\begin{equation}\label{MDP}
		%\RL{\cost}\quad
		\rho_\cost^\star\triangleq\min_{\pi\in\Pi_0}\rho_\cost(\pi),
		\tag*{$\RL{\cost}$}
		\end{equation}
		where $\rho_\cost(\pi)=\innerprod{\initial}{\val^\pi_\cost}$ is the \emph{total expected cost} of $\pi$. 
		
		For every policy $\pi$, we define the \emph{normalized state-action occupancy measure} $\mv_\pi\in\Delta_{\sspace\times\aspace}$, by
		$
		\mu_\pi(x,a) \triangleq (1-\gamma) \sum_{t=0}^\infty \gamma^t \Prob_{\initial}^{\pi}\left[x_t=x,a_t=a\right],
		$
		where $\Prob_{\initial}^{\pi}[\cdot]$ denotes the probability of an event when following $\pi$ starting from $x_0\sim\initial$.
		The occupancy measure can be interpreted as the discounted visitation frequency of state-action pairs. This allows us to write $\rho_{\cost}(\pi)=\inner{\mv_\pi}{\cost}{}$.
		
		The \emph{optimal value function} $\val_\cost^\star\in\ar^{|\sspace|}$ is defined by
		$
		V_\cost^\star(x) \triangleq \min_{\pi\in\Pi_0}V_\cost^\pi(x).
		$
		For clarity, when the cost $\cost$ is a parameterized function, written as $\cost_\weight$ for a parameter vector $\weight$, we will replace $\cost$ by $\weight$
		in the notation of the aforementioned quantities, e.g., $\rho_\weight(\pi)$, $\val^\pi_\weight$, etc.

		\subsection{Learning from Demonstrations}
		
		The goal of LfD is to learn a policy that outperforms the expert policy $\expert$ for an unknown true cost function $\cost_{\textup{true}}$. We assume that the learner is given only a finite set of truncated expert sample trajectories and is not allowed to interact or query the expert for more data while training.
		
		Although the MDP model is not known, we assume access to a \emph{generative-model oracle} which, given a state-action pair $(x,a)$, outputs the next state $x'\sim P(\cdot|x,a)$.
		Moreover we can sample $x_0\sim\initial$.
		This is also known as the simulator-defined MDP~\cite{Szorenyi:2014,Taleghan:2015}.
		
		To address the LfD problem, we adopt the AL formalism~\cite{Abbeel:2004,syed2008apprenticeship,Ho:2016,Ho:2016b}, which carries the assumption that $\true$ belongs to a class of cost functions $\mcf{C}$. We then seek a policy that performs better than the expert across $\mcf{C}$ by solving the following minimax optimization problem
		\begin{equation}\label{AL}
		%\LfD\quad
		\alpha^\star\triangleq
		\min_{\pi\in\Pi_0} \max_{\cost\in\mcf{C}}\rho_\cost(\pi)-\rho_\cost(\expert).
		%\min_\pi\delta_\mcf{C}(\pi,\expert),
		\tag*{$\LfD$}
		\end{equation}
		Equivalently, we can write $\alpha^\star=\min_\pi \delta_\mcf{C}(\pi,\expert)$, where
		$
		\delta_\mcf{C}(\pi,\expert)\triangleq\max_{\cost\in\mcf{C}} \big(\rho_\cost(\pi)-\rho_\cost(\expert)\big)
		$
		denotes the $\mcf{C}$-distance between $\pi$ and $\expert$~\cite{Ho:2016,Chen:2020a,Zhang:2020}.
		An optimal solution $\apprentice$ to $\LfD$ is called an \emph{apprentice policy} and satisfies $\rho_{\true}(\apprentice)\le\rho_{\true}(\expert)+\alpha^\star$ with $\alpha^\star$ being always nonpositive.
		
		Intuitively, the cost class $\mcf{C}$ distinguishes the expert from other policies. The maximization in~$\LfD$ assigns high total cost to non-expert policies and low total cost to $\expert$~\cite{Ho:2016}, while the minimization aims to find the policy that matches the expert as close as possible with respect to the $\mcf{C}$-distance.
		
		In this work, we assume that the true cost function can be  represented  as a convex combination  of  some  known features, where the true unknown weights specify how different desiderata should be traded-off.
		In particular, we consider the following cost function class~\cite{Syed:2007,syed2008apprenticeship,Ho:2016}
		$$
		\mcf{C}=\mathcal{C}_{\rm conv}\triangleq\{\cost_{\weight}\triangleq\sum_{i=1}^{n_{{c}}} w_i \cost_i \mid w_i\geq 0,\;\sum_{i=1}^{n_\cost}w_i=1\},
		$$	
		where $\{\cost_i\}_{i=1}^{n_\cost}\subset\Re^{\abs{\sspace}\abs{\aspace}}$ are fixed cost vectors, such that $\norm{\cost_i}_\infty \le 1$ for all $i=1,\ldots,n_\cost$.
		In this case, $\delta_{\mathcal{C}}(\pi,\pi_{\expert})=\max_{i\in[n_c]}\big(\rho_{\cost_i}(\pi)-\rho_{\cost_i}(\expert)\big)$.
		
		Note that this assumption is not necessarily restrictive as usually in practice the true cost function depends on just a few key properties, but the desirable weighting is unknown~\cite{Abbeel:2004}. Moreover, these features can be
		arbitrarily complex nonlinear functions and can be obtained via unsupervised learning
		from raw state observations~\cite{Brown:2020b,Chen:2020b}.
		
		We highlight that one can consider other linearly parameterized cost classes, e.g., $\mcf{C}_{\rm lin} = \{\sum_{i=1}^{n_\cost} w_i \psi_i \mid \norm{w}_2 \le 1\}$ \cite{Abbeel:2004}, leading to similar reasoning and analysis.

\subsection{Technique Overview}
 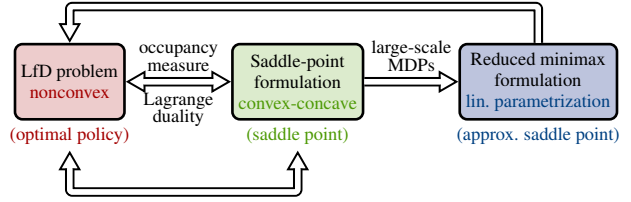
\begin{figure}
    \centering
	\begin{tikzpicture}
      \scriptsize
      % block nodes
      \node[block, align=center,fill=tangored!20] (lfd) {LfD problem \\ \textcolor{tangored}{nonconvex}};
      \node[block,right=1.4 of lfd.east, align=center,fill=tangogreen!20] (saddle) {Saddle-point \\ formulation \\ \textcolor{tangogreen}{convex-concave}};
      \node[block,right=1.25 of saddle.east,align=center,fill=ethblue!20] (reduced) {Reduced minimax \\ formulation \\ \textcolor{ethblue}{lin.\ parametrization}};
    
      % text nodes
      \node[below=0 of lfd] (text1) {\textcolor{tangored}{(optimal policy)}};
      \node[below=0 of saddle] (text2) {\textcolor{tangogreen}{(saddle point)}};
      \node[below=0 of reduced] (text3) {\textcolor{ethblue}{(approx.\ saddle point)}};
        
      % arrows
      \draw[vecArrowDouble] (lfd) -- node[above=.2em,align=center] {occupancy \\[-.2em] measure} node[below=.2em,align=center] {Lagrange \\[-.2em] duality} (saddle);
      \draw[vecArrow] (saddle) -- node[above=.2em,align=center] {large-scale \\[-.2em] MDPs} (reduced);
      \draw[vecArrow,rounded corners=1pt] (reduced) -- ++ (0,1) -| (lfd);
      \draw[vecArrowDouble,rounded corners=1pt] (text1) -- ++ (0,-.75) -| (text2);
    \end{tikzpicture}
	\caption{Main building blocks of our analysis.}
	\label{fig:diagram}
  \end{figure}                                                                  Before diving into technical details, we provide a technique overview by formalizing the main building blocks of our analysis, as described informally in the introduction. Figure~\ref{fig:diagram} illustrates the big picture behind our reduction. 

The main difficulty in $\LfD$ comes  from its nonconvex-nonconcave structure. In Section~\ref{sec:convex_optimization_view}, by using Lagrangian duality, we handle this difficulty by transforming the original nonconvex-nonconcave optimization program $\LfD$ into the bilinear saddle-point problem~(\ref{eq:full}). The new decision variables are occupancy measures $\mv$, value functions $\uv$ and cost weights $\weight$. The two formulations are equivalent, since the primal optimizer $\ma$ directly maps to that of the original problem $\apprentice$ and vice versa. Moreover, under the assumption of expert optimality the triplet $(\mv_{\expert},\mbf{V}^\star_{\mbf{w_{\textup{true}}}},\weight_{\textup{true}})$ is a saddle-point of~(\ref{eq:full}).  

For large-scale MDPs the saddle-point formulation~(\ref{eq:full}) is intractable. To mitigate this difficulty, in Section~\ref{sec:linearlyrelaxed} we consider the reduced saddle-point problem~(\ref{LRALP}) by introducing parameterized families $\{\mv_{\thv}\mid\thv\in\Theta\}$ and $\{\ul\mid\lv\in\Lambda\}$, for appropriately chosen parameter sets $\Theta$ and $\Lambda$. We then relate the saddle-point residual ${\epsilon}_{\text{sad}}(\boldsymbol{\theta},\boldsymbol{\lambda},\weight)$~(\ref{SPR}) of an approximate saddle-point $(\thv,\lv,\weight)$ to the suboptimality gap of the extracted policy $\pth$ for the original~$\LfD$. In particular we show that $\delta_{\mathcal{C}}(\pth,\expert)-\delta_{\mathcal{C}}(\apprentice,\expert)\le k\epsilon_{\textup{sad}}(\boldsymbol{\theta},\boldsymbol{\lambda},\weight)+\varepsilon_{\textup{approx}}$, where $\varepsilon_{\textup{approx}}$ is a measure of expressivity of the function approximators and $k\in\{1,3\}$. In this way, it is apparent that $\LfD$ is decoupled in two parts: minimization of ${\epsilon}_{\text{sad}}(\boldsymbol{\theta},\boldsymbol{\lambda},\weight)$  and function approximation. By choosing linear function approximators (LFA), we preserve the convexity-concavity of the Lagrangian and so we are able to design a scalable algorithm with sublinear convergence rate to a neighborhood of $\apprentice$, where the representation power of LFA determines the convergence error.

In particular, in Section~\ref{sec:Algorithm}, we present a  variance reduced stochastic mirror  descent algorithm~\cite{Nemirovski:2009}  under local norms and coordinate-wise gradient estimators~\cite{Carmon:2019,Jin:2020} to solve the $\LfD$ problem  with a generative-model oracle. After $N$ iterations and by using $m$ expert demonstrations we extract a policy $\pN$ such that $\rho_{\true}(\widehat{\pi}_N)-\rho_{\true}(\expert)\le\mathcal{O}\left(\tfrac{\log{\tfrac{1}{\delta}}}{\sqrt{N}}\right)+\mathcal{O}\left(\tfrac{\log{\tfrac{1}{\delta}}}{\sqrt{m}}\right)+\varepsilon_{\textup{approx}}$, with probability at least $1-\delta$. This translates to a sample complexity that scales linearly with the number of features  and does not depend on the number of states and actions.
 % \subsection{An Overview of the Main Results}
  
 % \GB{We promised to a reviewer that we will include this section.}

		%===================================================================================================================================================================================
		\section{The Convex Optimization View}\label{sec:convex_optimization_view}
		
		In this section, we revisit the LP approach to the LfD problem~\cite{syed2008apprenticeship} and introduce a novel linear duality framework, which will be later exploited to lay theoretical foundations for an efficient stochastic primal-dual algorithm. The details on the formulations and the proofs of this section can be found in Appendix~\ref{A}.
		\subsection{Primal-Dual Linear Programming Formulations}\label{subsec:primal-dual_LP}
		Assuming that  $\mcf{C}=\mcf{C}_{\rm conv}$, one can frame the  $\LfD$ problem as an LP over occupancy measures~\cite{syed2008apprenticeship}. We first review briefly the rationale behind the formulation.
		
		The set of occupancy measures can be characterized in terms of linear constraint satisfaction. To this aim, let
		$
		\mathfrak{F} \triangleq \left\lbrace \mv\in\Re^{|\sspace||\aspace|} \mid \op\mv=\initial, \; \mv\ge \boldsymbol{0} \right\rbrace,
		$
		where $\op\colon\ar^{|\sspace||\aspace|}\to\ar^{|\sspace|}$ is a linear operator given by
		$
		\op\mv=\frac{1}{1-\gamma}(\bmat-\gamma \pmat)^\intercal\mv.
		$
		Here, $\pmat$ is the vector form of $P$, i.e., $P_{(x,a),x'}\triangleq P(x'|x,a)$, and $\bmat$ is a binary matrix defined by $B_{(x,a),x'}\triangleq 1$ if $x=x'$, and $B_{(x,a),x'}\triangleq 0$ otherwise. 
		\begin{proposition}{\cite{Puterman:1994}}\label{eq:occup_meas_set}
			It holds that,
			$
			\mathfrak{F} = \left\lbrace \mv_\pi \mid \pi\in\Pi_0 \right\rbrace.
			$
			Indeed, for every $\pi\in\Pi_0$, we have that $\mv_\pi \in \mathfrak{F}$.
			Moreover, for every feasible solution $\mv\in \mathfrak{F}$, we can obtain a stationary Markov policy  $\pi_{\mv}\in\Pi_0$ by
			$
			\pi_{\mv}(a|x) \triangleq \frac{\mv(x,a)}{\sum_{a'\in\aspace}\mv(x,a')}.
			$
			Then, the induced occupancy measure is exactly $\mv$.
		\end{proposition}
		Using the epigraph reformulation and Proposition~\ref{eq:occup_meas_set}, it follows that $\LfD$ is equivalent to the following primal LP:
		\begin{equation}%\label{Primal}
		%\primal\quad& \left\{ 
		\begin{array}{ll}
		\min\limits_{(\mv,\varepsilon)} & \varepsilon \\
		\textup{s.t.} & \innerprod{\mv-\mv_\expert}{\cost_i}\le\varepsilon,\;i\in[n_c],\\
		& \mv\in\mathfrak{F}. %,\quad\varepsilon\in\Re.
		\end{array}
		%\right.
		\tag*{$\primal$}
		\end{equation}
		% \begin{equation}\label{Primal}
		%\textbf{P}_{\pi_{\textup{E}}}:\alpha^\star=\min_{\boldsymbol{\mu},\varepsilon}\{\varepsilon\mid %\innerprod{\boldsymbol{\mu}-\boldsymbol{\mu}_{\pi_{\textup{E}}}}{\mbf{c}_i}\le\varepsilon,\,i\in[n_c],\,\,\boldsymbol{\mu}\in\mathfrak{F}\}.
		% \end{equation}
		The linear program $\primal$ resembles the dual
		LP for solving forward MDPs~\cite{Puterman:1994}. Its optimization
		variables are occupancy measures and a maximum per-feature cost component.
		
		The constraints that define the set $\mathfrak{F}$ are also known as \emph{Bellman flow constraints} and ensure that $\mv$ is an occupancy measure generated by a stationary Markov policy. In particular, $(\boldsymbol{\mu}_{\pi_{\textup{A}}},\alpha^\star)$ is a primal optimizer. Conversely, by an optimal occupancy measure $\mv_{\textup{A}}$, an apprentice policy can be extracted as $\pi_{\mv_{\textup{A}}}$.

		The main objective of this subsection is to shed light to the dual of $\primal$ and interpret the dual optimizers.
		
		It holds that $\innerprod{\initial}{\uv}=\innerprod{\op\mv_\expert}{\uv}=\innerprod{\mv_\expert}{\op^*\uv}$, where $\op^*\colon\Re^{|\sspace|}\to\Re^{|\sspace||\aspace|}$, given by $\op^*\uv=\frac{1}{1-\gamma}(\bmat-\gamma\pmat)\uv$, is the adjoint operator of $\op$. We then get by standard linear duality that the dual LP is given by 
		\begin{equation}%\label{Dual}
		%\dual\quad & \left\{ 
		\begin{array}{ll}
		\max\limits_{(\uv,\weight)}&  \innerprod{\mv_\expert}{\op^*\uv-\cost_\weight}\\
		\textup{s.t.}& \cost_\weight-\op^*\uv \geq \mbf{0},\\[.25em]
		& \uv\in\Re^{|\sspace|},\quad\weight\in\Delta_{[n_c]}.
		\end{array}
		%\right.
		\tag*{$\dual$}
		\end{equation}
		%\begin{equation}\label{Primal}
		%\textbf{D}_{\pi_{\textup{E}}}:\,\,\,\,\alpha^\star=\max_{\mbf{w}\in\mathcal{W}}\max_{\mbf{u}\in\ar^X}\{\innerprod{\boldsymbol{\mu}_{\pi_{\text{E}}}}{\mbf{T^*}\mb%f{u}-\mbf{c_w}}\mid\mbf{c_w}-\boldsymbol{T}^*\mbf{u \geq 0}\}.
		%\end{equation}
		The inequality constraint in~$\dual$ is a relaxation of the Bellman optimality conditions. Therefore, $\dual$ resembles the primal LP for solving forward MDPs but its optimization variables are both value functions and cost weights.

		%By Proposition~\ref{eq:occup_meas_set}, $\pi_{\text{A}}$ is optimal for the $\LfD$~\textup{(\ref{AL})} with optimal value $\alpha^\star$ if and only if %$(\boldsymbol{\mu}_{\pi_{\textup{A}}},\alpha^\star)$ is optimal for $\primal$~(\ref{Primal}). 
		%Conversely, if $(\mbs{\mu}_{\textup{A}},\alpha^\star)$ is optimal for $\primal$~(\ref{Primal}), then $\pi_{\mbs{\mu}_{\textup{A}}}$ is optimal for %$\LfD$~(\ref{AL}).
		Next, we derive optimality conditions for our LP formulations. 
		We first assume that the expert is optimal for the true cost. In this case, the expert policy $\expert$ is an optimal solution to $\LfD$.
		\begin{lemma}\label{lemma:expert}
			Assume that $\expert$ is optimal for $\RL\true$.
			Then $\expert$ is optimal for $\LfD$ with optimal value $\alpha^\star=0$. Equivalently, $(\mv_\expert,0)$ is optimal for $\primal$.
		\end{lemma}
		The following Proposition gives necessary and sufficient conditions for dual optimality. The proof is based on Lemma~\ref{lemma:expert} (if $\expert$ is optimal, then $\alpha^\star=0$) and complementary slackness conditions for the LP approach to forward MDPs.
		\begin{proposition}[Optimal expert]\label{prop:primal-dual-exactoptimality}
			Assume that $\expert$ is optimal for $\RL\true$.
			A pair $(\uv_{\textup{A}},\weight_{\textup{A}})$ is optimal for $\dual$ if and only if $\weight_{\textup{A}}\in\Delta_{[n_c]}$, $\expert$ is optimal for $\RL{\cost_{\weight_{\textup{A}}}}$ and $\uv_{\textup{A}}=\mbf{V}_{\weight_{\textup{A}}}^\star$. In particular, $(\mbf{V}^\star_{\mbf{w_{\textup{true}}}},\mbf{w_{\textup{true}}})$ is an optimal solution to~$\dual$.
		\end{proposition}
		
		Proposition~\ref{prop:primal-dual-exactoptimality} states that, under the expert optimality assumption, the set of dual opimizers characterizes the set of solutions to the IRL problem, i.e, the set of costs in $\mathcal{C}_{\rm{conv}}$  for which the expert is optimal. Indeed, a weight vector $\weight_{\textup{A}}\in\Delta_{[n_u]}$ is dual optimal if and only if the expert policy $\expert$ is optimal for the forward RL problem with cost $\cost_{\weight_{\textup{A}}}$. In this case, $\uv_{\textup{A}}$ coincides with the corresponding optimal value function\footnote{To be precise, this is the case if $\initial\in\Re^{|\sspace|}_{++}$, otherwise they coincide $\initial$-almost surely.}. In particular, the true weights $\mbf{w_{\textup{true}}}$ and the true optimal value function $\mbf{V}^\star_{\mbf{w_{\textup{true}}}}$ are dual optimizers.
		
		Is is worth noting that the linear program~$\dual$ is different from the LP formulation of IRL proposed in~\cite{Ng:2000,Komanduru:2019}, which is based on dynamic programming principles. In contrast, our characterization of solutions to the inverse problem is based on duality arguments in the same line as~\cite{Ahuja:2001,Iyengar:2005,Pauwels:2016}. Moreover, the dual LP formulation~$\dual$ is independent of the complexity of the expert policy $\expert$, which can be even history-dependent. 
		
		In the general case of a suboptimal expert, we have the following necessary conditions for dual optimality.
		\begin{proposition}[Suboptimal expert]\label{prop:primal-dual-suboptimality}
			If $(\uv_{\textup{A}},\weight_{\textup{A}})$ is optimal for $\dual$, then any apprentice policy $\pi_{\text{A}}$ is optimal for $\RL{\cost_{\weight_{\textup{A}}}}$ and $\uv_{\textup{A}}=\mbf{V_{\weight_{\textup{A}}}^\star}$. 
		\end{proposition}
	
		Proposition~\ref{prop:primal-dual-suboptimality} states that, in general, the apprentice policy $\pi_{\textup{A}}$ is an optimal solution to the forward RL problem with cost $\cost_{\weight_\textup{A}}$, where $\weight_{\textup{A}}$ is an optimal dual variable. In addition, $\uv_{\textup{A}}$ coincides with the corresponding optimal value function. 
		
		A similar result to Proposition~\ref{prop:primal-dual-suboptimality} has been obtained in~\cite{Ho:2016b} for the convex-concave formulation of maximum causal entropy IRL, by using the Sion minimax theorem~\cite{Sion:1958}. 
		%However, in~\cite{Ho:2016b} the dual optimization variable is only the cost function.
		
		\subsection{Saddle-Point Formulation}
		A serious limitation of the primal and dual LP formulations in Section~\ref{subsec:primal-dual_LP} is that they have an intractable number of constraints, which in addition may be not satisfied when working with function approximators.
		To mitigate this difficulty, we propose a more tractable unconstrained formulation by using the Lagrangian 
		\[ %\begin{equation}\label{eq:Lagrangian}
		\mcf{L}(\mv,\uv,\weight)=\innerprod{\mv-\mv_\expert}{\cost_\weight-\op^*\uv}.
		\] %\end{equation}
		We can then frame the problem as a bilinear saddle-point problem
		\begin{equation}\label{eq:full}
	      \alpha^\star = \min_{\mv\in\Delta_{\sspace\times\aspace}} \,\max_{(\uv,\weight)\in\mathcal{U}\times\Delta_{[n_c]}} \mcf{L}(\mv,\uv,\weight),
		\end{equation}
		
		%\begin{equation}\label{eq:full}
	     % \alpha^\star = %\min_{\mv\in\Delta_{\sspace\times\aspace}} %\,\max_{\substack{\uv\in\mcf{U} \\ %\weight\in\Delta_{[n_c]}}} \mcf{L}(\mv,\uv,\weight),
	%	\end{equation}
		where
	$
		\mcf{U}\triangleq\{\uv\in\Re^{|\sspace|}\mid\norm{\uv}_\infty\le 1\}.
	$
		Note that we have added extra constraints $\mv\in\Delta_{\sspace\times\aspace}$ and $\uv\in\mcf{U}$ for the sake of analysis. These constraints do not change the problem optimality, but will considerably accelerate the convergence of the algorithm by considering smaller domains.
		
		Indeed, note that the constraint $\mv\in\Delta_{\sspace\times\aspace}$ is redundant since it is satisfied for all primal feasible solutions. Moreover, by Proposition~\ref{prop:primal-dual-suboptimality}, for all dual optimizers $(\uv_{\textup{A}},\weight_{\textup{A}})$ it holds that $\uv_{\textup{A}}=\val^\star_{\weight_{\textup{A}}}$. Thus, we obtain the bound $\norm{\uv_{\textup{A}}}_\infty\le 1$.  
		
		The next Corollary follows from Propositions~\ref{prop:primal-dual-exactoptimality}--\ref{prop:primal-dual-suboptimality}.
		\begin{corollary}\label{cor:saddle-points}
			Suppose that the expert is optimal for $\true$. Then, $(\mv_{\expert},\val^\star_{\true},\weight_{\textup{true}})$ is a saddle-point to the  minimax problem~(\ref{eq:full}). In the general case, for a saddle-point $(\mv_{\textup{A}},\uv_{\textup{A}},\weight_{\textup{A}})$ it holds that: (i) $\pi_{\mv_{\textup{A}}}$ is optimal for $\LfD$, and (ii) $\pi_{\ma}$ is optimal for $\cost_{\weight_{\textup{A}}}$ and $\uv_{\textup{A}}=\val^\star_{\weight_{\textup{A}}}$.
		\end{corollary}

		%Disucss over constrained linear program see Nachum.
		%The LP formulations are computational inefficient and require knowledge of the dynamics and full access to the expert policy $\pi_{\textup{E}}$. BY Lagrangian duality

		%=====================================================================================================================================================================
		\section{A Linearly-Relaxed Saddle-Point Problem}\label{sec:linearlyrelaxed}
		
		Optimizing directly over $\mv$ and $\uv$ is impractical since their dimensions scale linearly with the size of the state space. We reduce the order of complexity by limiting our search to linear subspaces defined by a small number of features. In particular, we consider the feature matrices $\mbs{\Phi}\in\Re^{|\sspace||\aspace|\times n_{\mu}}$ and $\mbs{\Psi}\in\Re^{|\sspace|\times n_u}$. We then assume that the decision variables are parameterized in the form $\mv_{\thv}=\mbs{\Phi}\thv$ and $\uv_{\lv}=\mbs{\Psi}\lv$, where $\thv\in\Theta$ and $\lv\in\Lambda$ are the parameters to learn. Moreover, $\Theta$ and $\Lambda$ are appropriately chosen parameter sets. We denote $\mathcal{W}\triangleq\Delta_{[n_c]}$ for brevity. 
		
		The corresponding \emph{linearly-relaxed} saddle-point formulation is
		\begin{equation}\label{LRALP}
		 \min_{\thv\in\Theta} \, \max_{\substack{\lv\in\Lambda \\ \weight\in\mcf{W}}} \rlag(\thv,\lv,\weight),
		\end{equation}
		where $\rlag\colon\Theta\times\Lambda\times\mcf{W}\to\Re$ it the \emph{reduced} bilinear Lagrangian
		$
		\rlag(\thv,\lv,\weight)\triangleq\innerprod{\mv_{\thv}-\mv_{\expert}}{\cost_{\weight}-\op^*\uv_{\lv}}.
		$
		
		We measure the quality of an approximate saddle-point $(\boldsymbol{\theta},\boldsymbol{\lambda},\mbf{w})$ by its \emph{saddle-point residual (SPR)} defined as
%	\begin{equation}\label{SPR}
%		    {\epsilon}_{\text{sad}}(\boldsymbol{\theta},\boldsym%bol{\lambda},\weight)=\max_{(\lv',\weight')\in\Lambda\times\mat%hcal{W}}\rlag(\thv,\lv',\weight') - %\min_{\thv'\in\Theta}\rlag(\thv',\lv,\weight).
%		\end{equation}

	    \begin{equation}\label{SPR}
		    {\epsilon}_{\text{sad}}(\boldsymbol{\theta},\boldsymbol{\lambda},\weight)=\max_{\substack{\lv'\in\Lambda \\ \weight'\in\mcf{W}}}\rlag(\thv,\lv',\weight') - \min_{\thv'\in\Theta}\rlag(\thv',\lv,\weight).
		\end{equation}

		We assume that every column of $\mbs{\Phi}$ belongs to $\Delta_{\sspace\times\aspace}$, and every column of $\mbs{\Psi}$ belongs to $\mcf{U}$. These conditions ensure that if we choose
		\begin{equation*}
		\Theta\triangleq \Delta_{[n_\mu]}, \;
		\Lambda\triangleq  \{\lv\in\Re^{n_u}\mid\norm{\lv}_2 \le \frac{\beta}{\sqrt{n_u}}\}, \;
		\beta\geq 2,
		\end{equation*}
		then $\mv_{\thv}\in\Delta_{\sspace\times\aspace}$ and $\norm{\uv_{\lv}}_{\infty}\le \beta$, for all $\thv\in\Theta$ and $\lv\in\Lambda$. Note that we have enlarged the original optimization domain for the $\uv$ variable. In addition, by the definition of $\mathcal{C}_{\rm conv}$, we get $\norm{\cw}_\infty\le 1$, for all $\weight\in\mathcal{W}$. These bounds will be used in the sequel in the algorithm design and its theoretical analysis.
		
		\subsection{From Approximate Saddle Points to Optimal Policies}\label{sec:linearly_relaxed}
		A crucial part in our setting is to connect the saddle-point residual ${\epsilon}_{\text{sad}}(\boldsymbol{\theta},\boldsymbol{\lambda},\weight)$ of a feasible solution $(\thv,\lv,\weight)$ for the linearly-relaxed saddle-point problem~(\ref{LRALP}) to the suboptimality gap $\delta_{\mathcal{C}}(\pi_{\mth},\expert)-\delta_{\mathcal{C}}(\apprentice,\expert)$ of the induced policy $\pi_{\mth}$. If small SPR implies small suboptimality gap, then we can first run any stochastic primal-dual algorithm with sublinear convergence rate to the objective in~(\ref{LRALP}), and then convert the obtained approximate saddle-point to a nearly optimal policy for the original $\LfD$ problem. 
		
		For the following results, let $\apprentice$ be an apprentice policy, i.e., $\apprentice$ is optimal for $\LfD$. The proofs can be found in Appendix~\ref{B}.
		Our first result makes no assumptions on the choice of features.
		\begin{proposition}\label{prop:strong}
			Let $(\thv,\lv,\weight)\in\Theta\times\Lambda\times\mathcal{W}$ be a feasible solution to~(\ref{LRALP}). Set $\pi_{\thv}=\pi_{\mth}$. It then holds that
			\[
			\delta_{\mathcal{C}}(\pth,\expert)-\delta_{\mathcal{C}}(\apprentice,\expert)\le\epsilon_{\textup{sad}}(\boldsymbol{\theta},\boldsymbol{\lambda},\weight)+\varepsilon_{\textup{approx},\thv},
			\]
			where the approximation error $\varepsilon_{\textup{approx},\thv}$ is no larger than
			$
			\tfrac{2}{1-\gamma}\big(\beta\min_{\thv'\in\Theta}\norm{\mv_{\thv'}-\mv_{\apprentice}}_1+\min_{\lv'\in\Lambda}\norm{\uv_{\lv'}-\val_{\cost_{i_{\thv}}}^{\pth}}_\infty\big),
			$
			and $i_{\thv}\triangleq\arg\max_{i\in[n_c]}\big(\rho_{\cost_i}(\pth)-\rho_{\cost_i}(\expert)\big).$
			In particular,
			$
			\rho_{\true}(\pth)-\rho_{\true}(\expert)\le\epsilon_{\textup{sad}}(\boldsymbol{\theta},\boldsymbol{\lambda},\weight)+\varepsilon_{\textup{approx},\thv}+\alpha^\star.
			$
			
		\end{proposition}

        By Proposition~\ref{prop:strong}, if $(\thv,\lv,\weight)$ is an $(\varepsilon,\delta)$-optimal saddle point to~(\ref{LRALP}), i.e., ${\epsilon}_{\textup{sad}}(\boldsymbol{\theta},\boldsymbol{\lambda},\weight)\le\varepsilon$ with probability at least $1-\delta$, then the associated induced policy $\pi_{\thv}$ is $(\varepsilon+\varepsilon_{\textup{approx},\thv}+\alpha^\star)$-optimal for the original $\LfD$ problem with high probability. Recall, that $\alpha^\star$ is always zero or negative. The term $\varepsilon_{\textup{approx},\thv}$ is a measure of expressiveness of the linear function approximators. The approximation error $\varepsilon_{\textup{approx},\thv}$ depends on how well $\{\mth:\thv\in\Theta\}$ approximates the apprentice occupany measure $\mv_{\apprentice}$ and how well $\{\ul:\lv\in\Lambda\}$ approximates the value function 
	    $\val_{\cost_{i_{\thv}}}^{\pth}$ of the extracted policy $\pth$ under the cost $\cost_{i_{\thv}}$. Note, however, that  $\val_{\cost_{i_{\thv}}}^{\pth}$ is not fixed before the learning process. Therefore, in order to guarantee a priori low approximation error, we need the optimal occupancy measure to be accurately representable by $\{\mth:\thv\in\Theta\}$, while $\{\ul:\lv\in\Lambda\}$ is required to accurately represent the value functions of non-optimal policies as well, under costs in $\{\cost_i:\,i\in [n_c]\}$. Borrowing the terminology from~\cite{Shariff:2020}, we require the occupancy measure features to be \emph{weak}, while the value function features to be \emph{strong}.
	    
	    An open question is whether we can relax the previously described notion of \emph{good features} for the value function approximation~\cite{Shariff:2020}.
	    We provide the following result in this direction. However, the analysis requires stronger conditions on the choice of features. In particular, we introduce the following assumption first studied in~\cite{Bas-Serrano:2020b}.
		
		\begin{assumption}[Coherence Assumption]\label{ass:coherence}
			The columns of $\psim$ are well-conditioned in the following sense: for every $\thv\in\Theta$ and $\uv\in \Re^{|\sspace|}$ with $\norm{\uv}_\infty\le 2$, there exists $\lv\in\Lambda$ such that
			$
			\innerprod{\initial-\op\mth}{\uv-\ul}=0.
			$
		\end{assumption}
		\begin{proposition}\label{prop:weak}
		Let Assumption~\ref{ass:coherence} hold. Let $(\thv,\lv,\weight)\in\Theta\times\Lambda\times\mathcal{W}$ be a feasible solution to~(\ref{LRALP}). Set $\pi_{\thv}=\pi_{\mth}$. It then holds that
		\[
		\rho_{\true}(\pth)-\rho_{\true}(\expert)\le3\epsilon_{\textup{sad}}(\boldsymbol{\theta},\boldsymbol{\lambda},\weight)+\varepsilon_{\textup{approx},w}+\alpha^\star,
		\]
		where the weak approximation error $\varepsilon_{\textup{approx},w}$ is no larger than
		$
		\tfrac{2}{1-\gamma}\Big(\beta\min_{\thv'\in\Theta}\norm{\mv_{\thv'}-\mv_{\apprentice}}_1+2\norm{\val^\star_{\wa}-\psim\lv^\star}_\infty\Big),
		$
		where $(\val_{\wa}^\star,\wa)$ is dual optimal for~$\dual$ and $\lv^\star$ is a dual optimizer of~(\ref{LRALP}).
		\end{proposition}
		Assumption~\ref{ass:coherence} is satisfied when the rows of $\psim$ are orthonormal. As it is apparent from the proof of Proposition~\ref{prop:weak}, Assumption~\ref{ass:coherence} controls the \emph{distribution mismatch} between ${\boldsymbol{\mu}}_\theta$ and ${\boldsymbol{\mu}}_{\pi_\theta}$. In general $\mth$ is not necessarily an occupancy measure generated by a policy. The term $\norm{\op\mth-\initial}_1$, i.e., the violation degree of the Bellman flow constraints is related to the quality of the extracted policy $\pth$. Assumption~\ref{ass:coherence} ensures that an approximate saddle point with small violation degree of the \emph{aggregated} Bellman flow constraints $\psim^\intercal(\op\mth-\initial)=0$, has also small violation degree for the original Bellman flow constraints and thus produces a \emph{good} policy.

		The approximation error term $\norm{\val^\star_{\wa}-\psim\lv^\star}_\infty$  can be bounded by approximate dynamic programming techniques~\cite{DeFarias:2003,DeFarias:2004,Shariff:2020}. For example, under the \emph{realizability assumption}~\cite{Chen:2018,Bas-Serrano:2020a},  we get that $\varepsilon_{\textup{approx},w}=0$. This is formally stated in Lemma~\ref{lemma:realizability} in Appendix~\ref{B}.

		Moreover, assuming that there exists a set of \emph{core} states whose features span those of other states, the weak value function features ensure a low approximation error $\norm{\val^\star_{\wa}-\psim\lv^\star}$~\cite{Laksh:2018,Shariff:2020}.
		
	 	\section{Algorithm and Finite-Sample Analysis}\label{sec:Algorithm}
		
		After the analysis of our saddle-point setup, the aim of this section is (i) to provide a computationally efficient stochastic primal-dual algorithm whose iteration and sample complexities do not grow with the size of state and action spaces, and (ii) to obtain explicit probabilistic performance bounds on the quality of the extracted policy with respect to the unknown true cost function. 
		
		\subsection{Stochastic Primal-Dual LfD Algorithm}
		Having formulated the linearly-relaxed saddle-point problem~(\ref{LRALP}) with a few variables and constraints, we will now propose an iterative stochastic approximation algorithm for the batch LfD problem. Assuming access to a generative-model oracle and a finite set of expert demonstrations, we propose a stochastic mirror descent primal-dual algorithm which keeps the advantages of simplicity of implementation, low memory requirements, and low computational complexity. 
		
		We will consider the slightly modified Lagrangian
		\begin{align*}
		\mlag(\thv,\lv,\weight)\triangleq&\innerprod{\mth-\mv_{\expert}}{\cost_{\weight}-\op^*\uv_{\lv}}\\
		&\phantom{{}=}\phantom{{}=}+C_{\beta,\gamma}\innerprod{\mv_{\thv}}{\mbf{1}}-2\innerprod{\weight}{\mbf{1}},
		\end{align*}
		where $\cbg\triangleq2\beta/(1-\gamma)$. Note that $\mlag$ differs from $\rlag$ up to a constant, since $\innerprod{\mth}{\mbf{1}}=\innerprod{\weight}{\mbf{1}}=1$, for all $\thv\in\Theta$ and $\weight\in\mathcal{W}$. Thus, using $\mlag$ instead of $\rlag$ does not change the saddle-point residuals. In addition, we have that $\nabla_{\thv}\mlag(\thv,\lv,\weight)\ge\mbf{0}$ and $\nabla_{\weight}\mlag(\thv,\lv,\weight)\le\mbf{0}$, for all $\thv\in\Theta$, $\lv\in\Lambda$ and $\weight\in\mathcal{W}$. The same is true for their stochastic gradients. We will need this property in order to derive finer error bounds by using local norm arguments. For a more detailed discussion, see  Section~\ref{sec:sample-complexity}.   
		
		In the forthcoming material we use the following notation. The cost matrix is denoted by $\cma\triangleq\begin{bmatrix}\cost_1&\ldots&\cost_{n_c} \end{bmatrix}$ and the feature matrices by $\phim\triangleq\begin{bmatrix}\phiv_1&\ldots&\phiv_{n_{\mu}}\end{bmatrix}$ and $\psim\triangleq\begin{bmatrix}\psiv_1&\ldots&\psiv_{n_u}\end{bmatrix}$. The $(x,a)$-th row of $\cma$ is denoted by $\cost_{(x,a)}$ and and the $x$-th row of $\psim$ by $\psiv_{x}$.

		Note that both the expert policy $\expert$ and the transition law $P$ are unknown and they do appear in the reduced Lagrangians $\rlag$ and $\mlag$. We will now explain how to tackle this difficulty.
		
		We define the \emph{feature expectation} vector $\fev\in\Re^{n_c}$ of the expert policy $\expert$ by 
		$
		\mbs{\rho}_{\cma}(\expert)\triangleq (\rho_{\cost_1}(\expert),\ldots,\rho_{\cost_{n_c}}(\expert))^\intercal.
		$
		Since in practice we do not have access to the whole policy $\expert$, but instead can observe a finite set of \iid sample trajectories $\mathcal{D}_{\expert}^{m,H}\triangleq\{(x_0^k,a_0^k,x_1^k,a_1^k,\ldots,x_H^k,a_H^k)\}_{k=1}^m\sim\expert$, we consider the empirical feature expectation vector $\efev$ by taking sample averages, i.e., for each $i=1,\ldots,n_c$,
		$
		\rho_{\cost_i}(\widehat{\expert})\triangleq (1-\gamma)\frac{1}{m} \sum_{t=0}^H \sum_{j=1}^m \gamma^t c_i(x_t^j,a_t^j).
		$
		So the final empirical Lagrangian that our algorithm optimizes is given by
		\begin{align*}
		\elag(\thv,\lv,\weight)\triangleq&\innerprod{\mth}{\cw-\op^*\ul+C_{\beta,\gamma}\cdot\mbf{1}}
		\\
		&\phantom{{}=}+\innerprod{\initial}{\ul}-\innerprod{\weight}{\efev+2\cdot\mbf{1}}.
		\end{align*}
		
		\begin{algorithm}[tb]
			\caption{Stochastic Primal-Dual LfD}
			\label{alg:main}
			\begin{algorithmic}[1]
				\STATE {\bfseries Input:} cost matrix $\mbf{C}$, feature matrices $\mbs{\Phi}$ and $\mbs{\Psi}$
				\STATE {\bfseries Input:} number of iterations $N$, step-size $\eta$, radius $\beta$
				\STATE {\bfseries Input:} expert demonstrations $\mathcal{D}_{\textup{E}}^{m,H}$, generative-model 
				\STATE Compute $\efev$ using expert demonstrations.
				\STATE Set $\thv_{1,i}=\frac{1}{n_{\mu}}$, $i\in [n_{\mu}]$, $\weight_{1,i}=\frac{1}{n_c}$, $i\in[n_c]$, $\lv=\mbf{0}$
				%	\STATE Set $\mv=\Phi\thv_1$.
				\FOR{$n=1,\ldots N-1$}
				\STATE \texttt{// $\thv$ gradient estimation}
				\STATE Sample $i\sim\rm{Unif}([n_{\mu}])$, $(x,a)\sim{\boldsymbol{\phi}}_{i}$, $y\sim P(\cdot|x,a)$
				\STATE Set
				
				$g_{n,\thv,j}=
				\begin{cases}
				\frac{n_{\mu}\left( (1-\gamma)\cost_{(x,a)}^\intercal\weight_n-(\mbs{\psi}_x-\gamma\mbs{\psi}_{y})^\intercal\lv_n+2\beta\right)}{1-\gamma},j=i\\
				0,\quad\mbox{otherwise}
				\end{cases}
				$
				\STATE \texttt{// $\lv$ gradient estimation}
				\STATE Sample $x'\sim\initial$, $i\sim\thv_n$, $(x,a)\sim\phiv_i$, $y\sim P(\cdot|x,a)$
				\STATE Set  $\mbf{g}_{n,\lv}=\mbs{\psi}_{x'}-\frac{\mbs{\psi}_{x}-\gamma \mbs{\psi}_{y}}{1-\gamma}$
				\STATE \texttt{// $\weight$ gradient estimation}
				\STATE Sample $i\sim\thv_n$, $(x,a)\sim\phiv_i$
				\STATE Set $\mbf{g}_{n,\weight}=\mbf{c}_{(x,a)}-\efev-2\cdot\mbf{1}$
				\STATE \texttt{// Stochastic mirror descent steps}
				\STATE  Update $\theta_{{n+1},j}\propto \theta_{n,j}e^{(-\eta g_{n,\thv,j})},\,\,j\in[n_{\mu}]$
				\STATE  Update $\lv_{n+1}=\Pi_{\Lambda}(\lv_n +\tfrac{\eta\beta^2}{n_u}\mbf{g}_{n,\lv})$
				\STATE  Update $w_{{n+1},j}\propto w_{n,j}e^{(\eta g_{n,\weight,j})},\,\,j\in[n_c]$
				%	\STATE Set $\mv=\mbs{\Phi}\thv_{n+1}$
				\ENDFOR
				\STATE Set ${\hat{\thv}}_N=\frac{1}{N}\sum_{n=1}^N\thv_n$
				\STATE {\bfseries Output:} $\hat{\pi}_N=\pi_{\mbs{\Phi}\hat{\thv}_N}$
				
			\end{algorithmic}
		\end{algorithm}
		
		Our optimization variable is $\mbf{z}=(\thv,\lv,\weight)\in\mbf{Z}$, where $\mbf{Z}\triangleq\Theta\times\Lambda\times\mathcal{W}$ is the decision space.
		The monotone operator $\mbf{G}(\mbf{z})\triangleq
		\begin{bmatrix}
		\nabla_{\thv}\elag(\zv)^\intercal		 &
		-\nabla_{\lv}\elag(\zv)^\intercal		 &
		-\nabla_{\weight}\elag(\zv)^\intercal	
		\end{bmatrix}^\intercal$ is given by
		\begin{equation*}
		\mbf{G}(\mbf{z})=
		\begin{pmatrix}
		\mbs{\Phi}^\intercal\left(\cost_{\weight}-\op^*\uv_{\lv}+C_{\beta,\gamma}\cdot\mbf{1}\right)		 \\
		\psim^\intercal(\op\mth-\initial)	 \\
		\efev-\cma^\intercal\mth+2\cdot\mbf{1}
		\end{pmatrix}.
		\end{equation*}
		Note that when $\ul$ is the value function of a policy, then the term $\cost_{\weight}-\op^*\uv_{\lv}$ is the corresponding \emph{advantage function}. Moreover, in general $\mth$ is not necessarily an occupancy generated by a policy. The term $\norm{\psim^\intercal(\op\mth-\initial)}_1$ is the violation degree of the aggregated Bellman flow constraints. 
		
		Although the dynamics are unknown, by having at our disposal the generative-model oracle we can compute cheap unbiased gradient estimates. Even in the case of known dynamics, we have expensive matrix computations and one can accelerate learning by randomization. The rationale behind the sampling procedure and the computation of the gradient estimates is similar to~\cite{Chen:2018}. Indeed, note that we can equivalently consider a stochastic saddle-point formulation, by writing the empirical Lagrangian $\elag$ in the following sample-friendly form
	
		$
		\elag(\zv)
		$
		
		$
		 =\athroisma{j}{1}{n_\mu}{\theta_j}\mathop{{}\Exp}_{\substack{(x,a)\sim\phiv_j \\ {y\sim P(\cdot|x,a)}}}\left[\frac{ (1-\gamma)\cost_{(x,a)}^\intercal\weight-(\mbs{\psi}_x-\gamma\mbs{\psi}_{y})^\intercal\lv+2\beta}{1-\gamma}\right]
		$
		
		$
		\phantom{{}=}+\Exp_{x\sim\initial}\left[\psiv_x^\intercal\lv\right]-\innerprod{\weight}{\efev+2\cdot\mbf{1}}.
		$
		
        In particular, in round $n$ we make three independent calls of the generative oracle and compute the unbiased stochastic gradient 
		$\mbf{g}_n=
		\begin{bmatrix}
		\mbf{g}_{n,\thv}^\intercal	 &
		-\mbf{g}_{n,\lv}^\intercal		 &
		-\mbf{g}_{n,\weight}^\intercal		
		\end{bmatrix}^\intercal
		$
		of $\mbf{G}(\mbf{z}_n)$, as described in Algorithm~\ref{alg:main}.  We defer the details to Appendix~\ref{C}. 
		
	   Note that each iteration costs $\mathcal{O}(1)$ sample generation. Therefore the iteration and the sample complexity of the algorithm coincide. Moreover, the computation of gradient with respect to $\lv$ and $\weight$ is $\mathcal{O}(n_u)$ and $\mathcal{O}(n_c)$, respectively. On the other hand the computation of gradient with respect to $\mu$ is $\mathcal{O}(1)$ since only one coordinate is updated per iteration.
		
		The update rule of primal-dual mirror descent is given by
		\[
		\mbf{z}_{n+1}=\arg\min_{\mbf{z}\in\mbf{Z}} \big(\innerprod{\mbf{g}_n}{\mbf{z}}+\tfrac{1}{\eta}B_{R}(\mbf{z}||\mbf{z}_n)\big),
		\] 
		where $R$ is the following distance-generating function 
		\[
		R(\mbf{z})=\sum_{i=1}^{n_{\mu}}\theta_i\log \theta_i+\frac{n_u}{2\beta^2}\norm{\lv}_2^2+\sum_{i=1}^{n_c}w_i\log w_i.
		\]
		The choice of Shannon entropy for the variables $\mv$ and $\weight$ which live in the probability simplex mitigates the effects of dimension. Similarly, the factor $\tfrac{n_u}{\beta^2}$ is chosen to make the size of Bregman divergence dimension-free~\cite{Cheng:2020}.
		
		The analytical form of the updates can be seen in Algorithm~\ref{alg:main}. Note once more that the offsets considered in the modified Lagrangian have no effect since the gradients are exp-transformed and the resulting distribution normalized. We denote by $\widehat{\mbf{z}}_N\triangleq\frac{1}{N}\athroisma{n}{1}{N}{\zv_n}$ the iterate average after $N$ iterations. Then, the output of the algorithm is the policy $\widehat{\pi}_N\triangleq\pi_{\mv_{\widehat{\thv}_N}}$.

		\subsection{Sample Complexity}\label{sec:sample-complexity}
		The following theorem gives the sample complexity of Algorithm~\ref{alg:main}. The proof can be found in Appendix~\ref{D} and is based on high-confidence
		regret bounds of mirror descent with local norm arguments~\cite{Carmon:2019,Jin:2020}, combined with Propositions~\ref{prop:strong}--\ref{prop:weak}. In Appendix~\ref{E} we provide preliminary empirical results on a simple tabular MDP in order to
		 illustrate our formulations and theoretical results.

		\begin{theorem}\label{th:strong}
		Let $\widehat{\pi}_N$ be the output of running Algorithm~\ref{alg:main} for $N=\max\left\{\mathcal{O}\left(\frac{\beta^2 n_\mu \log\left(\frac{1}{\delta}\right)}{(1-\gamma)^2\varepsilon^2}\right),\mathcal{O}\left(\frac{\beta\sqrt{n_\mu^3 \log\left(\frac{1}{\delta}\right)}}{(1-\gamma)\varepsilon}\right)\right\}$ iterations, with $m=\frac{8\log(\frac{4n_c}{\delta})}{\varepsilon^2}$ expert trajectories of length $H=\frac{1}{1-\gamma}\log(\frac{2}{\varepsilon})$, and learning rate $\eta=\frac{1-\gamma}{\beta\sqrt{N n_\mu}}$ . Then, with probability $1-\delta$ it holds that
		$
		\rho_{\true}(\widehat{\pi}_N)-\rho_{\true}(\expert)\le\varepsilon+\varepsilon_{\textup{approx},\widehat{\thv}_N}+\alpha^\star.
		$
		If in addition Assumption~\ref{ass:coherence} is satisfied,  then with probability $1-\delta$, it holds that
		$
		\rho_{\true}(\widehat{\pi}_N)-\rho_{\true}(\expert)\le\varepsilon+\varepsilon_{\textup{approx},w}+\alpha^\star.
		$
		\end{theorem}
		%	Note that because the definition of the total expected cost is scaled by a factor %$(1-\gamma)$, we finally have sample complexity %$N=\mathcal{O}(\frac{1}{(1-\gamma)^4\varepsilon^2})$ and expert sample complexity %$m=\mathcal{O}(\frac{1}{(1-\gamma)^2\varepsilon^2})$ for the conventional discounted MDP %setting. 
			
			Note that we do not have full access to the
			true dynamics and the expert policy but instead we can
			query the generative oracle and observe a batch
			finite set of truncated expert demonstrations. In our
			stochastic algorithm this is depicted to the fact that we replace the expert feature expectation vector $\fev$ by its empirical counterpart $\efev$, and we
			do not consider the  full gradient vector but instead we
			use unbiased gradient estimates. The estimation error is
			quantified in terms of concentration inequalities. Indeed
			for the sample average of the expert feature expectation vector we use a variant of the Hoeffding's
			inequality~\cite{Syed:2007}, and for the error due to the stochastic
			gradient we choose appropriate martingale concentration inequalities~\cite{McDiarmid:1998}.

		By using the modified Lagrangian $\mlag$ and making the unbiased gradient estimates of $\boldsymbol{\theta}\in\Delta_{[n_\mu]}$ and $ \weight\in\Delta_{[n_c]}$ nonnegative/nonpositive, our final bounds have a better dimension dependency. This trick, also used in~\cite{Jin:2020,Cheng:2020}, allows us to attain more refined convergence guarantees by exploiting the low variance bounds under the corresponding local norms (as opposed to the $\norm{\cdot}_\infty$-norm). We refer to Sec.~2.8 in~\cite{Shalev:2012} for the theory behind this technique. 
		 %We refer to Sec.~2.8 in (Shalev-Shwartz. Online learning) for the theory behind this technique. 
		 In our case, by replacing $\norm{\mathbf{g}_{n,\thv}}_\infty^2$ with the local norm $\norm{\mathbf{g}_{n,\thv}}_{\thv_n}^2\triangleq\sum_{i=1}^{n_\mu}\theta_{n,i} {g}_{n,\thv,i}^2$, we improve the regret bounds by a $n_\mu$-factor. 
		 
		 The parameter $\beta$ acts as a regularization in learning. If it is too small, the projection residuals $\min_{\lv\in\Lambda}\norm{\ul-\val_{\cost_{i_{\thv}}}^{\pth}}_\infty$ and $\norm{\val^\star_{\wa}-\psim\lv^\star}_\infty$ in the approximation error terms $\varepsilon_{\textup{approx},\thv}$ and $\varepsilon_{\textup{approx},w}$, respectively, are bigger. If it is too large, the learning becomes slower.

		 	\subsection{A No-Regret Online Learning View}
		 	
		 	In Appendix~\ref{F}, we argue that solving the online learning problem with decision set $\mbf{Z}=\Theta\times\Lambda\times\mathcal{W}$ and per-round loss function
		 	$
		    \ell_n(\zv) \triangleq \rlag(\thv,\lv_n,\weight_n) - \rlag(\thv_n,\lv,\weight)
		 	$
		 	%=& \innerprod{\mth-\mexp}{\cost_{\weight_n}-\op^*\uv_{\lv_n}}\\
		 	%	&\phantom{{}=}-\innerprod{\mv_{\thv_n}-\mexp}{\cw-\op^*\ul},
		 	is equivalent to solving the linearly-relaxed saddle-point problem~(\ref{LRALP}). 
		 	
		 	It is worth noting that this online learning approach differs from the one in~\cite{ross2011reduction} where interaction with the expert is required. It is also different from the game-theoretic approach in~\cite{Syed:2007} where the forward RL problem has to be solved repeatedly.
		 
		 \subsection{Linear Function Approximators}
		 We are motivated to work with occupancy measures (OMs) $\mv$ instead of policies $\pi$ because of linearity and flexibility of $\primal$ and $\dual$. We highlight that Propositions~\ref{prop:strong}--\ref{prop:weak} hold even for nonlinear parameterizations $\{\mth\mid\thv\in\Theta\}$ and $\{\ul\mid\lv\in\Lambda\}$, as long as $\mth\in\mathcal{M}$ and $\norm{\ul}_\infty\le\beta$ hold. Thus, the LfD problem is decoupled in two parts: minimization of SPR and function approximation. For the first part, we employ modern large-scale stochastic optimization methods while the second part can be quantified independently of the learning process. We choose LFA to preserve the convexity-concavity of the Lagrangian and design a scalable algorithm with theoretical guarantees (sublinear convergence rate), though at a cost of a potential approximation bias. Indeed, LFA require a careful choice of features. One can use prior knowledge of the MDP to choose appropriate basis functions for the value functions, e.g., in our case the value function basis $\{\psiv_i=\val_{\cost_i}^\star\mid i\in [n_c]\}$) satisfies the realizability assumption. On the other hand the choice of OM features is trickier since only a restrictive class of policies can be represented~\cite{Banijamali:2019}. One can use OMs of policies extracted from ``heuristic'' methods or policies provided from multiple ``cheaper'' suboptimal experts (where one can sample a lot of rollouts) and use Algorithm~\ref{alg:main} to improve upon them. We hope that our techniques will be useful for future algorithm designers and will lay foundations for more exhaustive research in this direction. In Appendix~\ref{G} we point out a
		 few interesting directions.

		 \textbf{Acknowledgements}
		 
		 This project has received funding from the European Research Council 
		 (ERC) under the European Union’s Horizon 2020 research and innovation 
		 programme grant agreement OCAL, No.\ 787845.

		%====================================================================================================================================================================================	
		\bibliography{refs}
		\bibliographystyle{icml2021}
		
	\onecolumn
	\newpage
	\appendix 
       \section{Derivations and Proofs of Section~\ref{sec:convex_optimization_view}}~\label{A}
       For the following results, we will assume that $\initial\in\Re^{|\sspace|}_{++}$. We also set $\mathcal{W}\triangleq\Delta_{[n_c]}$.
       \subsection{Primal LP formulation~$\primal$}
       We have
       \begin{align}
       \alpha^\star=&\min_{\pi}\max_{c\in\mathcal{C}}\rho_{\cost}(\pi)-\rho_{\cost}(\expert)\tag{definition of $\LfD$}\\
       =& \min_{\pi}\max_{i\in[n_c]}\rho_{\cost_i}(\pi)-\rho_{\cost_i}(\expert)\tag{since $\mathcal{C}=\mathcal{C}_{\rm{conv}}$}\\
       =& \min_{\pi}\max_{i\in[n_c]}\innerprod{\mv_{\pi}-\mexp}{\cost_i} \nonumber\\
       =& \min_{\mv\in\mathfrak{F}}\max_{i\in[n_c]} \innerprod{\mv-\mexp}{\cost_i} \tag{by Proposition~\ref{eq:occup_meas_set}}\\
       =& \min_{\mv,\varepsilon}\{\varepsilon\mid\innerprod{\mv-\mexp}{\cost_i}\le\varepsilon,\,i\in[n_c],\,\,\mv\in\mathfrak{F}\}.\tag{by an epigraphic transformation}
       \end{align}
       \subsection{Dual LP formulation~$\dual$}
        We have
        \begin{align}
        \alpha^\star=& \min_{\mv,\varepsilon}\{\varepsilon\mid\innerprod{\mv-\mexp}{\cost_i}\le\varepsilon,\,i\in[n_c],\,\,\mv\in\mathfrak{F}\}\tag{definition of $\primal$}\\
        =&\min_{\mv\ge\boldsymbol{0},\varepsilon}\max_{\weight\geq\boldsymbol{0},\uv}\quad\varepsilon+\sum_{i=1}^{n_c}w_i\left(\innerprod{\mv-\mexp}{\cost_i}-\varepsilon\right)+\innerprod{\initial-\op\mv}{\uv}\tag{introduce Lagrange multipliers $\weight$ and $\uv$}\\
        =&\min_{\mv\ge\boldsymbol{0},\varepsilon}\max_{\weight\geq\boldsymbol{0},\uv}\quad\varepsilon\left( 1-\sum_{i=1}^{n_c}w_i\right)+\innerprod{\mv-\mexp}{\cw-\op^*\uv}\tag{since $\op\mexp=\initial$}\\
      =& \max_{\weight\geq\boldsymbol{0},\uv} \min_{\mv\ge\boldsymbol{0},\varepsilon}\quad\varepsilon\left( 1-\sum_{i=1}^{n_c}w_i\right)+\innerprod{\mv-\mexp}{\cw-\op^*\uv}\tag{linear duality}\\
      =& \max_{\weight\geq\boldsymbol{0},\uv}\left\{\innerprod{\mexp}{\op^*\uv-\cw}\mid\cw-\op^*\uv\geq\boldsymbol{0},\,\,\athroisma{i}{1}{n_c}{w_i}=1\right\}.\nonumber
        \end{align}

       \subsection{Proof of Proposition~\ref{prop:primal-dual-exactoptimality}}
     \textbf{Proposition} $\mbf{2}$ (Optimal expert).
\emph{Assume that $\expert$ is optimal for $\RL\true$.
 A pair $(\uv_{\textup{A}},\weight_{\textup{A}})$ is optimal for $\dual$ if and only if $\weight_{\textup{A}}\in\Delta_{[n_c]}$, $\expert$ is optimal for $\RL{\cost_{\weight_{\textup{A}}}}$ and $\uv_{\textup{A}}=\mbf{V}_{\weight_{\textup{A}}}^\star$. In particular, $(\mbf{V}^\star_{\mbf{w_{\textup{true}}}},\mbf{w_{\textup{true}}})$ is an optimal solution to~$\dual$.}
  
 \textbf{Lemma} $\mbf{1}$.
\emph{Assume that $\expert$ is optimal for $\RL\true$.
	Then $\expert$ is optimal for $\LfD$ with optimal value $\alpha^\star=0$. Equivalently, $(\mv_\expert,0)$ is optimal for $\primal$.}
	
	\begin{proof}
	It holds that $\alpha^\star\le\delta_{\mathcal{C}}(\expert,\expert)=0$. Assume for the sake of contradiction that $\alpha^\star<0$. Then $\rho_{\true}(\apprentice)-\rho_{\true}(\expert)\le\delta_{\mathcal{C}}(\apprentice,\expert)=\alpha^\star<0$, which contradicts the assumption of optimality of $\expert$ with respect to $\true$. Thus, $\alpha^\star=0$ and in particular $\expert$ is optimal for $\LfD$.
	\end{proof}

We recall the LP approach to MDPs~\cite{Puterman:1994,Bertsekas:2012}. Let $\cost\in\Re^{|\sspace||\aspace|}$ be a cost function. The forward RL problem $\RL{\cost}$ is equivalent to the following linear programs\footnote{Note that usually in the literature the primal LP is~(\ref{MDP-dual}).}
\begin{align}
\rho_{\cost}^\star=&\min_{\mv\in\Re^{|\sspace|\aspace||}}\{\innerprod{\mv}{\cost}\mid\op\mv=\initial,\,\,\,\mv\geq\boldsymbol{0}\}\tag{MDP-P$_{\cost}$}\label{MDP-primal}\\
=& \max_{\uv\in\Re^{|\sspace|}}\{\innerprod{\initial}{\uv}\mid\cost-\op^*\uv\geq\boldsymbol{0}\} \tag{MDP-D$_{\cost}$} \label{MDP-dual}
\end{align}
We have that if $\pi^\star$ is an optimal policy for $\RL{\cost}$, then $\mv_{\pi^\star}$ is optimal  for~(\ref{MDP-primal}) and converselly if $\mv^\star$ is optimal for~(\ref{MDP-primal}), then $\pi_{\mv^\star}$ is an optimal policy for $\RL{\cost}$. Moreover, the optimal value function $\val_{\cost}^\star$ is an optimal solution to~(\ref{MDP-dual}) and it is the unique optimizer when $\initial\in\Re^{|\sspace|}_{++}$.
\begin{proofof}{Proof of Proposition~\ref{prop:primal-dual-exactoptimality}} The proof is based on linear duality and complementary slackness conditions of optimality. Assume first that $(\ua,\wa)$ is optimal for $\dual$. Then,
\begin{align}
\cost_{\wa}-\op^*\ua\geq\boldsymbol{0} \label{eq:feasibility},\quad\wa\in&\,\mathcal{W},\\
\innerprod{\mv_{\expert}}{\op^*\ua-\cost_{\wa}}=\alpha^\star=&\,\,0, \label{eq:optimality}
\end{align}	
where~(\ref{eq:feasibility}) holds because $(\ua,\wa)$ is feasible to $\dual$, and~(\ref{eq:optimality}) holds by optimality and Lemma~\ref{lemma:expert}.
Equation~(\ref{eq:feasibility}) states that $\ua$ is feasbile for~(\ref{MDP-dual}) with cost $\cost=\cost_{\wa}$. Moreover, $\mexp$ is feasible for~(\ref{MDP-primal}) with cost $\cost=\cost_{\wa}$. Therefore, 
\begin{equation}\label{eq:complementary}
\innerprod{\initial}{\ua}\le\rho^\star_{\wa}\le\innerprod{\mexp}{\cost_{\wa}}.
\end{equation}
However,  by~(\ref{eq:optimality}) we get that $\innerprod{\initial}{\ua}=\innerprod{\op\mexp}{\ua}=\innerprod{\mexp}{\op^*\ua}=\innerprod{\mexp}{\cost_{\wa}}$.
Thus, by~(\ref{eq:complementary}) we conclude that $\mexp$ is optimal for~(\ref{MDP-primal}) with cost $\cost=\cost_{\wa}$ and $\ua$ is optimal for~(\ref{MDP-dual}) with cost $\cost=\cost_{\wa}$. Thus $\expert$ is optimal for $\RL{\cost_{\wa}}$ and $\ua=\val^\star_{\cost_{\wa}}$.

Converselly, assume that $\wa\in\mathcal{W}$, $\expert$ is optimal for $\RL{\cost_{\weight_{\textup{A}}}}$, and $\ua=\val^\star_{\wa}$. Then, we have that $\mexp$ is optimal for~(\ref{MDP-primal}) with cost $\cost=\cost_{\wa}$, and $\ua$ is optimal for~(\ref{MDP-dual}) with $\cost=\cost_{\wa}$. Therefore, 
\begin{align}
\cost_{\wa}-\op^*\ua\geq&\,\,\boldsymbol{0} \label{eq:feasibility2},\\
\innerprod{\initial}{\ua}=&\innerprod{\mexp}{\cost_{\wa}},\label{eq:optimality2}
\end{align}
where~(\ref{eq:feasibility2}) holds because $\ua$ is feasible for~(\ref{MDP-dual}) with cost $\cost=\cost_{\wa}$, and~(\ref{eq:optimality2}) holds by primal-dual optimality of $(\mexp,\ua)$.
From~(\ref{eq:feasibility2}), we get that $(\ua,\wa)$ is feasible to~$\dual$ and from~(\ref{eq:optimality2}), we get that $\innerprod{\mv_{\expert}}{\op^*\ua-\cost_{\wa}}=0$. Since by Lemma~\ref{lemma:expert}, $\alpha^\star=0$ we conclude that $(\ua,\wa)$ is optimal for~$\dual$.
\end{proofof}

\subsection{Proof of Proposition~\ref{prop:primal-dual-suboptimality}}

\textbf{Proposition} $\mbf{3}$ (Suboptimal expert).
\emph{	If $\mv_{\textup{A}}$ is optimal for $\primal$ and $(\uv_{\textup{A}},\weight_{\textup{A}})$ is optimal for $\dual$, then any apprentice policy $\pi_{\text{A}}$ is optimal for $\RL{\cost_{\weight_{\textup{A}}}}$ and $\uv_{\textup{A}}=\mbf{V_{\weight_{\textup{A}}}^\star}$. }

\begin{proof}
	Similarly to the reasoning of the proof of Proposition~\ref{prop:primal-dual-exactoptimality}, we need to show that
	\begin{align}
	\cost_{\wa}-\op^*\ua&\geq\boldsymbol{0},\tag{dual feasibility to~(\ref{MDP-dual}) with $\cost=\cost_{\wa}$}\\
	\innerprod{\mpa}{\cost_{\wa}-\op^*\ua}&=0, \tag{complementarity}
	\end{align}
	The dual feasibility follows from the fact that $(\ua,\wa)$ is feasible to~$\dual$. Moreover, we have
	\begin{align}
	0\le&\innerprod{\mpa}{\cost_{\wa}-\op^*\ua}\tag{since $\cost_{\wa}-\op^*\ua\geq\boldsymbol{0}$ } \\
	=& \innerprod{\mpa}{\cost_{\wa}}-\innerprod{\mexp}{\op^*\ua}\tag{since $\op\mpa=\op\mexp=\initial$}\\
	\le&\innerprod{\mexp}{\cost_{\wa}}+\alpha^\star-\innerprod{\mexp}{\op^*\ua}\tag{since $\apprentice$ is optimal for $\LfD$}\\
	=&\,\,\alpha^\star-\alpha^\star=0.\tag{since $(\ua,\wa)$ is optimal for~$\dual$}
	\end{align} 
	So complementarity is satisfied. 
\end{proof}
	
\subsection{Proof of Corollary~\ref{cor:saddle-points}}
	 \textbf{Corollary} $\mbf{1}$.
	\emph{	Suppose that the expert is optimal for $\true$. Then, $(\mv_{\expert},\val^\star_{\true},\weight_{\textup{true}})$ is a saddle-point to the  minimax problem~(\ref{eq:full}). In the general case, for a saddle-point $(\mv_{\textup{A}},\uv_{\textup{A}},\weight_{\textup{A}})$ it holds that: (i) $\pi_{\mv_{\textup{A}}}$ is optimal for $\LfD$, and (ii) $\pi_{\ma}$ is optimal for $\cost_{\weight_{\textup{A}}}$ and $\uv_{\textup{A}}=\val^\star_{\weight_{\textup{A}}}$.}
		
	\begin{proof}
		The result follows from Propositions~\ref{prop:primal-dual-exactoptimality}--\ref{prop:primal-dual-suboptimality} by noting that $(\ma,\ua,\wa)$ is a saddle point of the full minimax problem~(\ref{eq:full}) if and only if $(\ma,\alpha^\star)$ is optimal for the primal linear program $\primal$ and $(\ua,\wa)$ is optimal for the dual linear program $\dual$.
	\end{proof}

    \section{Proofs of Section~\ref{sec:linearly_relaxed} and Related Results}\label{B}
For the following results, let $\apprentice$ be an apprentice policy, i.e., $\apprentice$ is optimal for $\LfD$.
    \subsection{Proof of Proposition~\ref{prop:strong}}
    
  \textbf{Proposition} $\mbf{4.}$ 
 \emph{ Let $(\thv,\lv,\weight)\in\Theta\times\Lambda\times\mathcal{W}$ be a feasible solution to~(\ref{LRALP}). Set $\pi_{\thv}=\pi_{\mth}$. It then holds that
  \[
  \delta_{\mathcal{C}}(\pth,\expert)-\delta_{\mathcal{C}}(\apprentice,\expert)\le\epsilon_{\textup{sad}}(\boldsymbol{\theta},\boldsymbol{\lambda},\weight)+\varepsilon_{\textup{approx},\thv},
  \]
  where the approximation error $\varepsilon_{\textup{approx},\thv}$ is no larger than
  $
  \tfrac{2}{1-\gamma}\big(\beta\min_{\thv'\in\Theta}\norm{\mv_{\thv'}-\mv_{\apprentice}}_1+\min_{\lv'\in\Lambda}\norm{\uv_{\lv'}-\val_{\cost_{i_{\thv}}}^{\pth}}_\infty\big),
  $
  and $i_{\thv}\triangleq\arg\max_{i\in[n_c]}\big(\rho_{\cost_i}(\pth)-\rho_{\cost_i}(\expert)\big).$
  In particular,
  $
  \rho_{\true}(\pth)-\rho_{\true}(\expert)\le\epsilon_{\textup{sad}}(\boldsymbol{\theta},\boldsymbol{\lambda},\weight)+\varepsilon_{\textup{approx},\thv}+\alpha^\star.
  $}
  \bigskip

    In the remainder of the paper, for any policy $\pi\in\Pi_0$, we define the matrix $\mmat_\pi\in\Re^{|\sspace|\times|\sspace||\aspace|}$ that encodes $\pi$ by $M_{\pi,x',(x,a)}\triangleq\pi(a|x)$, if $x=x'$ and $M_{\pi,x',(x,a)}\triangleq0$, otherwise. Moreover, for any cost function $\cost$, we denote by $\mbf{c}_\pi\in\Re^{|\sspace|}$ the average cost under policy $\pi$, i.e., $c_\pi(x)\triangleq\sum_{a}\pi(a|x)c(x,a)$. Finally, let $\pmat_\pi\in \Re^{|\sspace|\times|\sspace|}$ be the state-transition of running policy $\pi$, i.e., $P_{\pi,x,x'}\triangleq \sum_a \pi(a|x)P(x'|x,a)$.
    
    \bigskip
    We will need the following Lemma.
    \begin{lemma}\label{lemma:vector_form}
    	Let $\mv\in\Delta_{\sspace\times\aspace}$ (not necessarily an occupany measure) and let $\pi_{\mv}\in\Pi_0$ be the induced policy given by $
    	\pi_{\mv}(a|x) \triangleq \frac{\mv(x,a)}{\sum_{a'}\mv(x,a')}$. For any cost vector $\cost\in\Re^{|\sspace||\aspace|}$ it holds that $\innerprod{\mv}{\cost-\op^*\val_{\cost}^{\pi_{\mv}}}=0$. 
     \end{lemma}
     \begin{proof}
     	We have that $\val_{\cost}^{\pi_{\mv}}$ is the unique fixed point of the following Bellman equation~\cite{Puterman:1994},
     	\begin{equation}\label{eq:Bellman}
     	V_{\cost}^{\pi_{\mv}}(x)=\sum_{a}\pi_{\mv}(a|x)\left((1-\gamma)c(x,a)+\gamma\sum_{x\prime}P(x'|x,a)V_{\cost}^{\pi_{\mv}}(x')\right).
     	\end{equation}
     	Note that~(\ref{eq:Bellman}) can be written equivalently in vector form as $\cost_{\pem}-\frac{1}{1-\gamma}(\mbf{I}-\gamma\pmat_{\pem})\val_{\cost}^{\pem}=0$. We then get $\innerprod{\bmat^\intercal\mv}{\mmat_{\pem}(\cost-\op^*\val^{\pem}_{\cost})}=0$, since $\pmat_{\pem}=\mmat_{\pem}\pmat$, $\mbf{I}=\mmat_{\pem}\bmat$ and $\cost_{\pem}=\mmat_{\pem}\cost$. Finally the result follows by using that $\mmat_{\pem}^\intercal\bmat^\intercal\mv=\mv$.
     	
     \end{proof}
     We are now ready to prove Proposition~\ref{prop:strong}.
     \begin{proofof}{Proof of Proposition~\ref{prop:strong}}
     Let $\eith\in\Re^{n_c}$ be the $i_{\thv}$-th basis vector. We then have,
        \begin{align*}
     	\lag(\mth,\val_{\cith}^{\pth},\eith)-\lag(\mpa,\ul,\weight)&=\innerprod{\mth-\mexp}{\cith-\op^*\weirdv}-\innerprod{\mpa-\mexp}{\cw-\op^*\ul}\\
     	&= \innerprod{\initial}{\weirdv}-\innerprod{\mexp}{\cith}-\innerprod{\mpa-\mexp}{\cw}\\
     	&=\underbrace{\left(\weirdrho-\rho_{\cith}(\expert)\right)}_{=\delta_{\mathcal{C}}(\pth,\expert)}-\underbrace{\left(\rho_{\cw}(\apprentice)-\rho_{\cw}(\expert)\right)}_{\le\delta_{\mathcal{C}}(\apprentice,\expert)}\\
     	&\geq \delta_{\mathcal{C}}(\pth,\expert)-\delta_{\mathcal{C}}(\apprentice,\expert),
     	\end{align*}
     	where the second line follows because $\innerprod{\mth}{\cith-\op^*\weirdv}=0$ by Lemma~\ref{lemma:vector_form}, and $\innerprod{\mpa-\mexp}{\op^*\ul}=0$, since $\op\mpa=\op\mexp=\initial$.
     	
     	All in all, we have that
     	\begin{equation}\label{eq:first_opt_gap}
     	\delta_{\mathcal{C}}(\pth,\expert)-\delta_{\mathcal{C}}(\apprentice,\expert)\le \lag(\mth,\val_{\cith}^{\pth},\eith)-\lag(\mpa,\ul,\weight).
     	\end{equation}
     	By using a simple rearrangement trick we get
     		\begin{equation}\label{eq:second_opt_gap}
     		\delta_{\mathcal{C}}(\pth,\expert)-\delta_{\mathcal{C}}(\apprentice,\expert)\le \underbrace{\lag(\mth,\val_{\cith}^{\pth},\eith)-\lag(\mpa,\ul,\weight)-\epsilon_{\textup{sad}}(\thv,\lv,\weight)}_{\triangleq\varepsilon_{\textup{approx},\thv}}+\epsilon_{\textup{sad}}(\thv,\lv,\weight)
     		\end{equation}
     		It remains to bound the term $\varepsilon_{\textup{approx},\thv}$. We set $\zv'=(\thv',\lv',\weight')$ and $\mbf{Z}=\Theta\times\Lambda\times\mathcal{W}$ for brevity. We then have
     		\begin{align}
     		\varepsilon_{\textup{approx},\thv}\ =& \min_{\zv'\in\mbf{Z}}\lag(\mth,\weirdv,\eith)-\lag(\mpa,\ul,\weight)-\lag(\mth,\uv_{\lv^\prime},\weight^\prime)+\lag(\mv_{\thv^\prime},\uv_{\lv},\weight)\nonumber\\\nonumber
     		=& \min_{\zv'\in\mbf{Z}} \innerprod{\mth-\mexp}{\cith-\cost_{{\weight}^\prime}}+\innerprod{\op\mth-\initial}{\uv_{\lv^\prime}-\weirdv}+\innerprod{\mv_{\thv^\prime}-\mpa}{\cith-\op^*\ul}\\ \nonumber
     		\le&\min_{\zv'\in\mbf{Z}} \underbrace{\norm{\mth-\mexp}_1}_{\le 2}\norm{\cith-\cost_{{\weight}^\prime}}_\infty+\underbrace{\norm{\op\mth-\initial}_1}_{\le\frac{2}{1-\gamma}}\norm{\uv_{\lv^\prime}-\weirdv}_\infty+\norm{\mv_{\thv^\prime}-\mpa}_1\underbrace{\norm{\cith-\op^*\ul}_\infty}_{\le\frac{2\beta}{1-\gamma}}\\ 
           \le&\frac{2\beta}{1-\gamma}\min_{\thv^\prime\in\Theta}\norm{\mv_{\thv^\prime}-\mpa}_1+\frac{2}{1-\gamma}\min_{\lv^\prime\in\Lambda}\norm{\uv_{\lv^\prime}-\weirdv}_\infty. \label{eq:third_opt_gap}
     		\end{align}
     		The bounds on the norms involved in the first inequality have been obtained by using the triangle inequality and that $\norm{\mv}_1=1$, for all $\mu\in\Delta_{\sspace\times\aspace}$, $\mv_{\thv^\prime}\in\Delta_{\sspace\times\aspace}$, for all $\thv^\prime\in\Theta$, $\norm{\initial}_1=1$, $\norm{\bmat}_\infty=\norm{\pmat}_\infty=1$, $\norm{\uv_{\lv^\prime}}_\infty\le\beta$,  for all $\lv^\prime\in\Lambda$, and $\norm{\cost_{\weight^\prime}}_\infty\le 1$, for all $\weight^\prime\in\mathcal{W}$. The detailed computations can be found in Lemmas~\ref{lemma:bounds}--\ref{lemma:bounds_operators}. For the last inequality, note that obviously $\min_{\weight^\prime\in\mathcal{W}}\norm{\cith-\cost_{{\weight^\prime}}}_\infty=0$.
     		
     		By combining~(\ref{eq:second_opt_gap}) and~(\ref{eq:third_opt_gap}) we conclude the proof.
     \end{proofof}
     \subsection{Corollaries Related to Proposition~\ref{prop:strong}}
     By modifying analogously the proof of Proposition~\ref{prop:strong}, we get the following tighter bound.
     \begin{corollary}\label{cor:tighter}
     	 Let $(\thv,\lv,\weight)\in\Theta\times\Lambda\times\mathcal{W}$ be a feasible solution to~(\ref{LRALP}). Set $\pi_{\thv}=\pi_{\mth}$. It then holds that
     	 \[
     	 \delta_{\mathcal{C}}(\pth,\expert)-\delta_{\mathcal{C}}(\apprentice,\expert)\le \rlag(\thv,\lv_{\thv},\eith)-\rlag(\thv^\star,\lv,\weight)+\varepsilon_{\textup{approx},\thv},
     	 \]
     	 where the approximation error $\varepsilon_{\textup{approx},\thv}$ is no larger than
     	 $
     	 \tfrac{2}{1-\gamma}\big(\beta\min_{\thv\in\Theta}\norm{\mth-\mv_{\apprentice}}_1+\min_{\lv\in\Lambda}\norm{\ul-\val_{\cost_{i_{\thv}}}^{\pth}}_\infty\big),
     	 $
     	 and 
     	 $$
     	 \thv^\star\triangleq \arg\min_{\thv^\prime\in\Theta}\norm{\mu_{\thv^\prime}-\mpa}_1,\quad
     	 {\lv}_{\thv}\triangleq \arg\min_{\lv^\prime\in\Lambda}\norm{\uv_{\lambda^\prime}-\weirdv}_\infty,\quad
     	 i_{\thv}\triangleq \arg\max_{i\in[n_c]}\big(\rho_{\cost_i}(\pth)-\rho_{\cost_i}(\expert)\big).
     	 $$
     	 In particular,
     	 $
     	 \rho_{\true}(\pth)-\rho_{\true}(\expert)\le\rlag(\thv,\lv_{\thv},\eith)-\rlag(\thv^\star,\lv,\weight)+\varepsilon_{\textup{approx},\thv}+\alpha^\star.
     	 $
     \end{corollary} 
     When the occupancy measure features can accurately represent the expert occupancy measure $\mexp$ instead of the apprentice occupancy measure $\mv_{\apprentice}$, we may use alternatively the following result.
     
     \begin{corollary}
      Let $(\thv,\lv,\weight)\in\Theta\times\Lambda\times\mathcal{W}$ be a feasible solution to~(\ref{LRALP}). Set $\pi_{\thv}=\pi_{\mth}$. It then holds that
      \[
      \delta_{\mathcal{C}}(\pth,\expert)\le\epsilon_{\textup{sad}}(\boldsymbol{\theta},\boldsymbol{\lambda},\weight)+\varepsilon_{\textup{approx},\textup{E},\thv},
      \]
      where the approximation error $\varepsilon_{\textup{approx},\textup{E},\thv}$ is no larger than
      $
      \tfrac{2}{1-\gamma}\big(\beta\min_{\thv'\in\Theta}\norm{\mv_{\thv'}-\mv_{\expert}}_1+\min_{\lv'\in\Lambda}\norm{\uv_{\lv'}-\val_{\cost_{i_{\thv}}}^{\pth}}_\infty\big),
      $
      and $i_{\thv}\triangleq\arg\max_{i\in[n_c]}\big(\rho_{\cost_i}(\pth)-\rho_{\cost_i}(\expert)\big).$
      In particular,
      $
      \rho_{\true}(\pth)-\rho_{\true}(\expert)\le\epsilon_{\textup{sad}}(\boldsymbol{\theta},\boldsymbol{\lambda},\weight)+\varepsilon_{\textup{approx},\textup{E},\thv}.
      $
      \end{corollary}
    
     \begin{proof}
     	The proof is the same as the proof of Proposition~\ref{prop:strong} by replacing $\apprentice$ with $\expert$.
     \end{proof}
     \subsection{Proof of Proposition~\ref{prop:weak}}
\textbf{Proposition} $\mbf{5}.$
\emph{	Let Assumption~\ref{ass:coherence} hold. Let $(\thv,\lv,\weight)\in\Theta\times\Lambda\times\mathcal{W}$ be a feasible solution to~(\ref{LRALP}). Set $\pi_{\thv}=\pi_{\mth}$. It then holds that
	\[
	\rho_{\true}(\pth)-\rho_{\true}(\expert)\le3\epsilon_{\textup{sad}}(\boldsymbol{\theta},\boldsymbol{\lambda},\weight)+\varepsilon_{\textup{approx},w}+\alpha^\star,
	\]
	where the weak approximation error $\varepsilon_{\textup{approx},w}$ is no larger than
	$
	\tfrac{2}{1-\gamma}\Big(\beta\min_{\thv'\in\Theta}\norm{\mv_{\thv'}-\mv_{\apprentice}}_1+2\norm{\val^\star_{\wa}-\psim\lv^\star}_\infty\Big),
	$
	where $(\val_{\wa}^\star,\wa)$ is dual optimal for~$\dual$ and $\lv^\star$ is a dual optimizer of~(\ref{LRALP}).}

The proof of Proposition~\ref{prop:weak} is divided in several parts. We first note that even when a probability distribution $\mv\in\Delta_{\sspace\times\aspace}$ is not an occupancy measure, i.e., it does not satisfy the Bellman flow constraints, it can still  generate a policy $\pem$. The following lemma states that the \emph{distribution mismatch} between $\mv$ and  $\mv_{\pem}$ depends on the violation degree of the Bellman flow constraints.

\begin{lemma}\label{lemma:constraint_violation}
Let $\mv\in\Delta_{\sspace\times\aspace}$ and let $\pem$ be the induced policy. It holds that
$
\norm{\mv-\mpm}_1\le\norm{\initial-\op\mv}_1.
$
	
\end{lemma}
\begin{proof}
	We have that
	\begin{align}
	\mpm^\intercal&=(1-\gamma)\athroisma{t}{0}{\infty}{\gamma^t\initial^\intercal\mmat_{\pem}(\pmat\mmat_{\pem})^t}\\
	&= (1-\gamma)\athroisma{t}{0}{\infty}{\gamma^t(\initial-\op\mv+\op\mv)^\intercal\mmat_{\pem}(\pmat\mmat_{\pem})^t}\\
	&= (1-\gamma)\athroisma{t}{0}{\infty}{\gamma^t(\initial-\op\mv)^\intercal\mmat_{\pem}(\pmat\mmat_{\pem})^t}+\athroisma{t}{0}{\infty}{\gamma^t\mv^\intercal(\pmat\mmat_{\pem})^t}-\athroisma{t}{0}{\infty}{\gamma^{t+1}\mv^\intercal(\pmat\mmat_{\pem})^{t+1}}\\
	&= (1-\gamma)\athroisma{t}{0}{\infty}{\gamma^t(\initial-\op\mv)^\intercal\mmat_{\pem}(\pmat\mmat_{\pem})^t}+\mv^\intercal,
	\end{align}
	where in the third equality we used that $\op\mv=\frac{1}{1-\gamma}(\bmat-\gamma\pmat)^\intercal\mv$ and that $\mmat_{\pem}^\intercal\bmat^\intercal\mv=\mv$.
	Therefore,
	\begin{align}
	\norm{\mv-\mpm}_1&\le\norm{(1-\gamma)\athroisma{t}{0}{\infty}{\gamma^t}(\mmat_{\pem}^\intercal\pmat^\intercal)^t\mmat_{\pem}^\intercal(\initial-\op\mv)}_1\\
	&\le\norm{\mmat_{\pem}}_\infty^t\norm{\pmat}_\infty\norm{\mmat_{\pem}}_\infty\norm{\initial-\op\mv}_1\\
	&= \norm{\initial-\op\mv}_1,
	\end{align}
	where we used that $\norm{\pmat}_\infty=\norm{\mmat_{\pem}}_\infty=1$.
\end{proof}
We will also need the following lemma.
\begin{lemma}\label{lemma:first:bound}
	Let $(\thv,\lv,\weight)\in\Theta\times\Lambda\times\mathcal{W}$ be a feasible solution to~(\ref{LRALP}) and let $\pth=\pi_{\mth}$. It holds that
	\[
	\rho_{\true}(\pth)-\rho_{\true}(\expert)\le\sad+2\norm{\op\mth-\initial}_1+\frac{2\beta}{1-\gamma}\min_{\thv^\prime\in\Theta}\norm{\mv_{\thv^\prime}-\mv_{\apprentice}}_1
	+\alpha^\star.
	\]
	\end{lemma}
\begin{proof}
	We fix a $\lv_0\in\Lambda$ such that $\norm{\lv_0}_2\le\frac{1}{\sqrt{n_u}}$. By assumption the columns of $\psim$ belong to $\mathcal{U}$. Thus, $\norm{\ulz}_\infty\le 1$. Moreover,
	\begin{align}
	\lag(\mth,\ulz,\wtrue)-\lag(\mpa,\ul,\weight)=&\innerprod{\mth-\mexp}{\true-\op^*\ulz}-\innerprod{\mpa-\mexp}{\cw-\op^*\ul}\\
	\geq&\innerprod{\mth}{\true}-\rho_{\true}(\expert)-\innerprod{\op\mth-\initial}{\ulz}-\alpha^\star, \label{lastone}
	\end{align}
	where in the last inequality we used that $\innerprod{\mpa-\mexp}{\cw}\le\delta_{\mathcal{C}}(\apprentice,\expert)=\alpha^\star$, and that $\op\mpa=\op\mexp=\initial$. Now since $\innerprod{\op\mth-\initial}{\ulz}\le\norm{\op\mth-\initial}_1\norm{\ulz}_\infty$ and $\norm{\ulz}_\infty\le 1$, we get by~(\ref{lastone}) that
	\begin{equation}\label{eq:first}
	\innerprod{\mth}{\true}-\rho_{\true}(\expert)\le\lag(\mth,\ulz,\wtrue)-\lag(\mpa,\ul,\weight)+\norm{\op\mth-\initial}_1+\alpha^\star.
	\end{equation}
	Moreover, we have the bound
\begin{align}
\innerprod{\mth}{\true}-\rho_{\true}(\pth) &= \innerprod{\mth-\mv_{\pth}}{\true}\\
&\le \norm{\mth-\mv_{\pth}}_1 \norm{\true}_\infty\\
&\le \norm{\op\mth-\initial}_1, \label{eq:second}
\end{align}
where the last inequality follows by Lemma~\ref{lemma:constraint_violation} and the bound $\norm{\true}_\infty \le1$.

So by~(\ref{eq:first}) and (\ref{eq:second}) we get
\begin{equation}\label{eq:third}
\rho_{\true}(\pth)-\rho_{\true}(\expert)\le\lag(\mth,\ulz,\wtrue)-\lag(\mpa,\ul,\weight)+2\norm{\op\mth-\initial}_1+\alpha^\star.
\end{equation}
Next, we apply once more the following rearrangement trick
\begin{equation}\label{eq:fourth}
\lag(\mth,\ulz,\wtrue)-\lag(\mpa,\ul,\weight)=\underbrace{\lag(\mth,\ulz,\wtrue)-\lag(\mpa,\ul,\weight)-\sad}_{\triangleq(i)}+\sad.
\end{equation}
It remains to bound the term $(i)$. By similar calculations to the ones in the proof of Proposition~\ref{prop:strong}, we get
\begin{equation}\label{eq:fifth}
(i)\le\frac{2\beta}{1-\gamma}\min_{\theta^\prime\in\Theta}\norm{\mv_{\thv^\prime}-\mpa}_1+2\underbrace{\min_{\weight^\prime\in\mathcal{W}}\norm{\cw-\true}_\infty}_{=0}+\frac{1}{1-\gamma}\underbrace{\min_{\lv^\prime\in\Lambda}\norm{\uv_{\lv^\prime}-\ulz}_\infty}_{=0}.
\end{equation}
By combining~(\ref{eq:third}), (\ref{eq:fourth}) and (\ref{eq:fifth}), we conclude the proof.
\end{proof}
Finally the following lemma bounds the term $\norm{\op\mth-\initial}_1$.
\begin{lemma}\label{lemma:second-bound}
Let Assumption~\ref{ass:coherence} hold. It holds that 
$
\norm{\op\mth-\initial}_1\le\sad+\frac{2}{1-\gamma}\norm{\val^\star_{\wa}-\psim\lv^\star}_\infty,
$ 
where $\lv^\star$ is a dual optimizer for~(\ref{LRALP}).
\end{lemma}
\begin{proof}
	Let $\ut=\arg\max_{\norm{\uv}_\infty\le 2}\innerprod{\op\mth-\initial}{-\uv}$. By the Coherence Assumption~\ref{ass:coherence}, there exists $\lt\in\Lambda$, such that $\innerprod{\op\mth-\initial}{\uv_{\lt}-\ut}=0$. We then have,
	\begin{align*}
	2\norm{\op\mth-\initial}_1=&\max_{\norm{\uv}_\infty\le 2}\innerprod{\op\mth-\initial}{-\uv}\\
	=& \innerprod{\op\mth-\initial}{-\ut}\\
	=& \innerprod{\op\mth-\initial}{-\uv_{\lt}}\\
	=& \innerprod{\op\mth-\initial}{\val^\star_{\wa}-\uv_{\lt}}+\innerprod{\op\mth-\initial}{\val^\star_{\wa}}\\
	\le&  \innerprod{\op\mth-\initial}{\val^\star_{\wa}-\uv_{\lt}}+\norm{\op\mth-\initial}_1,
	\end{align*}
	where in the last inequality we used that $\norm{\val^\star_{\wa}}_\infty\le 1$. Therefore,
	\begin{equation}\label{eq:35}
	\norm{\op\mth-\initial}_1\le\innerprod{\op\mth-\initial}{\val^\star_{\wa}-\uv_{\lt}}.
	\end{equation}
	
	Now let $(\lv^\star,\weight^\star)$ be a dual optimal solution to the linearly-relaxed saddle-point formulation~(\ref{LRALP}). We can then write,
	\begin{align}\label{eq:}
\innerprod{\op\mth-\initial}{\val^\star_{\wa}-\uv_{\lt}}=&\rlag(\thv,\lt,\weight^\star)-\lag(\mth,\val^\star_{\wa},\weight^\star)\\
=&  \underbrace{\rlag(\thv,\lt,\weight^\star)-\rlag(\thv,\lv^\star,\weight^\star)}_{\le\sad}+\underbrace{\rlag(\thv,\lv^\star,\weight^\star)-\lag(\mth,\val^\star_{\wa},\weight^\star)}_{\le\frac{2}{1-\gamma}\norm{\val^\star_{\wa}-\psim\lv^\star}_\infty}.\label{eq:37}
	\end{align}
	We will now explain the bounds in~(\ref{eq:37}). For the first term, we have
	\begin{align*}
	\rlag(\thv,\lt,\weight^\star)-\rlag(\thv,\lv^\star,\weight^\star)\le&\max_{(\lv^\prime,\weight^\prime)\in\Lambda\times\mathcal{W}}\rlag(\thv,\lv^\prime,\weight^\prime)-\min_{\thv^\prime\in\Theta}\rlag(\thv^\prime,\lv^\star,\weight^\star)\\
	\le&\sad.
    \end{align*}
    For the second term we have
    \begin{align*}
    \rlag(\thv,\lv^\star,\weight^\star)-\lag(\mth,\val^\star_{\wa},\weight^\star)=&\innerprod{\op\mth-\initial}{\val^\star_{\wa}-\psim\lv^\star}\\
    \le& \underbrace{\norm{\op\mth-\initial}_1}_{\le\frac{2}{1-\gamma}}\norm{\val^\star_{\wa}-\psim\lv^\star}_\infty.
    \end{align*}
    
Putting together the bounds~(\ref{eq:35}) and (\ref{eq:37}) concludes the proof.	
\end{proof}
	
\begin{proofof}{Proof of Proposition~\ref{prop:weak}}
It follows by Lemma~\ref{lemma:first:bound} and \ref{lemma:second-bound}.	
\end{proofof}

     \subsection{Corollaries Related to Proposition~\ref{prop:weak}}
     When the occupancy measure features can accurately represent the expert occupancy measure $\mexp$ instead of the apprentice occupancy measure $\mv_{\apprentice}$, we may use alternatively the following result.

   \begin{corollary}
     	Let Assumption~\ref{ass:coherence} hold. Let $(\thv,\lv,\weight)\in\Theta\times\Lambda\times\mathcal{W}$ be a feasible solution to~(\ref{LRALP}). Set $\pi_{\thv}=\pi_{\mth}$. It then holds that
     	\[
     	\rho_{\true}(\pth)-\rho_{\true}(\expert)\le3\epsilon_{\textup{sad}}(\boldsymbol{\theta},\boldsymbol{\lambda},\weight)+\varepsilon_{\textup{approx},\textup{E},w},
     	\]
     	where the weak approximation error $\varepsilon_{\textup{approx},\textup{E},w}$ is no larger than
     	$
     	\tfrac{2}{1-\gamma}\Big(\beta\min_{\thv'\in\Theta}\norm{\mv_{\thv'}-\mv_{\expert}}_1+2\norm{\val^\star_{\wa}-\psim\lv^\star}_\infty\Big),
     	$
     	where $(\val_{\wa}^\star,\wa)$ is dual optimal for~$\dual$ and $\lv^\star$ is a dual optimizer of~(\ref{LRALP}).
    \end{corollary}
     \begin{proof}
     	In the proof of Lemma~\ref{lemma:first:bound} replace $\apprentice$ with $\expert$.
     \end{proof}
     
       The following Lemma shows that the realizability assumption implies that the weak approximation error $\varepsilon_{\textup{approx},w}$ is zero.
   
    	\begin{lemma}[Realizability]\label{lemma:realizability}
	 Let $(\uv_{\textup{A}},\weight_{\textup{A}})$ be an optimizer for the dual linear program~$\dual$. Assume that the realizability assumption holds, i.e., there exist $\tha\in\Theta$ and $\la\in\Lambda$ such that, $\mv_{\apprentice}=\phim\tha$ and $\ua=\psim\la$. Then, $(\tha,\la,\wa)$ is a saddle point of~(\ref{LRALP}). In particular $\varepsilon_{\textup{approx},w}=0$.
	\end{lemma}
    	 \begin{proof}
    	 	We need to show that 
    	 	\begin{equation}\label{eq:needtoshow}
    	 	\max_{\lv\in\Lambda}\max_{\weight\in\mathcal{W}}\rlag(\tha,\lv,\weight)=\rlag(\tha,\la,\wa)=\min_{\thv\in\Theta}\rlag(\thv,\la,\wa).
    	 	\end{equation}
    	 By the optimality of $(\mv_{\apprentice},\ua,\wa)$ for the full minimax problem~(\ref{eq:full}), we get that
    	 \begin{equation}\label{eq:star}
    	 \rlag(\tha,\la,\wa)=\lag(\mv_{\apprentice},\ua,\wa)=\alpha^\star.
    	 \end{equation}
    	 Moreover, for all $\thv\in\Theta$
    	 \begin{align}
    	 \rlag(\thv,\la,\wa)=&\innerprod{\mth-\mv_{\expert}}{\cost_{\wa}-\op^*\ua} \nonumber\\
    	 =&\underbrace{\innerprod{\mth}{\cost_{\wa}-\op^*\ua}}_{\geq 0}+\underbrace{\innerprod{\mv_{\expert}}{\op^*\ua-\cost_{\wa}}}_{=\alpha^\star} 
    	 \geq  \alpha^\star. \label{eq:bigger}
    	 \end{align}
    	 Note that by Proposition~\ref{prop:primal-dual-suboptimality}, $\ua$ is feasible for~(\ref{MDP-dual}) with cost function $\cost=\cost_{\wa}$. Therefore, $\cost_{\wa}-\op^*\ua\geq 0$. Moreover, $\innerprod{\mv_{\expert}}{\op^*\ua-\cost_{\wa}}=\alpha^\star$, since $(\ua,\wa)$ is optimal for $\dual$.
    	 
    	 In addition, for all $\lv\in\Lambda$ and $\weight\in\mcf{W}$
    	\begin{align}
    	 \rlag(\tha,\lv,\weight)=&\innerprod{\mv_{\apprentice}-\mv_{\expert}}{\cw-\op^*\ul}\\
    	 =& \innerprod{\mv_{\apprentice}-\mv_{\expert}}{\cw} \label{eq:zero}\\ 
    	 \le& \max_{\weight'\in\mathcal{W}}\innerprod{\mv_{\apprentice}-\mv_{\expert}}{\cost_{\weight'}} \\
    	 =&\,\,\delta_{\mathcal{C}}(\apprentice,\expert)=\alpha^\star \label{eq:ap:opt},
    	 \end{align}
    	 where~(\ref{eq:zero}) holds because $\op\mv_{\apprentice}=\op\mv_{\expert}=\initial$, and~(\ref{eq:ap:opt}) holds because $\apprentice$ is optimal for $\LfD$. 
    	 
    	 Finally, by~(\ref{eq:star}), (\ref{eq:bigger}) and~(\ref{eq:ap:opt}) we conclude that~(\ref{eq:needtoshow}) holds.
    	\end{proof}
    \section{Gradient Estimators and their Properties}\label{C}
    We will describe formally the unbiased gradient estimates used in Algorithm~\ref{alg:main} and study their properties.
    Our optimization variable is $\mbf{z}=(\thv,\lv,\weight)\in\mbf{Z}$, where $\mbf{Z}\triangleq\Theta\times\Lambda\times\mathcal{W}$ is the decision space.
    The empirical Lagrangian that we will optimize is
    	$
    	\elag(\thv,\lv,\weight)\triangleq\innerprod{\mth}{\cw-\op^*\ul+C_{\beta,\gamma}\cdot\mbf{1}}
    	+\innerprod{\initial}{\ul}-\innerprod{\weight}{\efev+2\cdot\mbf{1}}.
    	$
    	
    	\begin{definition}[Gradient estimators]\label{def:stochastic_gradients} We define $\gv\triangleq
    		\begin{bmatrix}
    		\gth^\intercal	 &
    		-\gl^\intercal		 &
    		-\gw^\intercal		
    		\end{bmatrix}^\intercal
    		$
    		as follows.
    		\begin{enumerate}

    		\item Let $\hat{\xi}\triangleq(\hi,\hx,\ha,\hy)\in\widehat{\Xi}\triangleq[n_\mu]\times\sspace\times\aspace\times\sspace$ such that 
    			$
    			\hi\sim\rm{Unif}([n_\mu])$, $(\hx,\ha)\sim\phiv_{\hi}$, $\hy\sim P(\cdot|\hx,\ha)$.
    			
    			We define the $\thv$-\emph{gradient estimator} $\gth:\mbf{Z}\times\widehat{\Xi}\rightarrow\ar^{n_\mu}$ by $$\gv_{\thv,i}(\zv,\hat{\xi})\triangleq
    			\begin{cases}
    			\frac{n_{\mu}\left( (1-\gamma)\cost_{(\hx,\ha)}^\intercal\weight-(\mbs{\psi}_{\hx}-\gamma\mbs{\psi}_{\hy})^\intercal\lv+2\beta\right)}{1-\gamma},\quad i=\hi,\\
    			0,\quad\mbox{\textup{otherwise}}.
    			\end{cases}
    			$$
    			
    	\item		Let $\tilde{\xi}\triangleq(\ti,\tx,\ta,\ty,\txp)\in\tilde{\Xi}\triangleq[n_\mu]\times\sspace\times\aspace\times\sspace\times\sspace$ such that 
    				$
    				\ti\sim\thv$, $(\tx,\ta)\sim\phiv_{\ti}$, $\ty\sim P(\cdot|\tx,\ta)$, $\txp\sim\initial$. We define the $\lv$-\emph{gradient estimator} $\gl:\mbf{Z}\times\tilde{\Xi}\rightarrow\ar^{n_u}$ by $\gl(\zv,\tilde{\xi})\triangleq\mbs{\psi}_{\txp}-\frac{\mbs{\psi}_{\tx}-\gamma \mbs{\psi}_{\ty}}{1-\gamma}$.
    				
    					\item		Let $\bar{\xi}\triangleq(\bi,\bx,\ba)\in\bar{\Xi}\triangleq[n_\mu]\times\sspace\times\aspace$ such that 
    					$
    					\bi\sim\thv$, $(\bx,\ba)\sim\phiv_{\bi}$. We define the $\weight$-\emph{gradient estimator} $\gw:\mbf{Z}\times\bar{\Xi}\rightarrow\ar^{n_c}$ by $\gw(\zv,\bar{\xi})\triangleq \cost_{(\bx,\ba)}-\efev-2\cdot\mbf{1}$.
    		\end{enumerate}

    	\end{definition}
    	
    \begin{lemma}\label{lemma:bounds}
   	For all $(\thv,\lv,\weight)\in\Theta\times\Lambda\times\mathcal{W}$ it holds that $\mth\in\Delta_{\sspace\times\aspace}$, $\norm{\ul}_\infty\le\beta$ and $\norm{\cw}_\infty\le 1$.
   \end{lemma} 
   \begin{proof}
   	We will show the first inequality. By assumption $\norm{\psiv_i}_\infty\le 1$, for all $i\in[n_u]$. Thus, $\norm{\ul}_\infty=\norm{\athroisma{i}{1}{n_u}{\lambda_i\psiv_i}}_\infty\le\norm{\lv}_1\max_{i\in[n_u]}\norm{\psiv_i}_\infty\le\beta$, where we used that $\norm{\lv}_1\le\sqrt{n_u}\norm{\lv}_2\le\beta$.  
   \end{proof}
   \begin{lemma}\label{lemma:bounds_operators}
   	For all $(\thv,\lv,\weight)\in\Theta\times\Lambda\times\mathcal{W}$ it holds that $\norm{\cw-\op^*\ul}_\infty\le\cbg\triangleq\frac{2\beta}{1-\gamma}$ and $\norm{\op\mth-\initial}_1\le\frac{2}{1-\gamma}$.
   	\end{lemma}
   	\begin{proof}
   		Indeed, for all $\lv\in\Lambda$ and $\weight\in\mathcal{W}$, we have that
   		\begin{align*}
   		\norm{\cw-\op^*\ul}_\infty=\norm{\cw-\frac{1}{1-\gamma}\left(\bmat-\gamma\pmat\right)\ul}_\infty
   		\le \norm{\cw}_\infty+\frac{1}{1-\gamma}\left(\norm{\bmat}_\infty+\gamma\norm{\pmat}_\infty\right)\norm{\ul}_\infty\le\cbg,
   		\end{align*} 
   		where we used the triangle inequality, Lemma~\ref{lemma:bounds} and that $\norm{\bmat}_\infty=\norm{\pmat}_\infty=1$. Similarly, for all $\thv\in\Theta$, we have
   		\begin{align*}
   	     \norm{\op\mth-\initial}_1=\norm{\initial-\frac{1}{1-\gamma}(\bmat-\gamma\pmat)^\intercal\mth}_1
   		\le \norm{\initial}_1+\frac{1}{1-\gamma}(\norm{\bmat}_\infty+\gamma\norm{\pmat}_\infty)\norm{\mth}_1\le\frac{2}{1-\gamma}.
   		\end{align*} 
   	\end{proof}
   	\begin{lemma}\label{lemma:nonnegative_gradients}
   	It holds that $\nabla_{\thv}\elag(\zv)\geq\mbf{0}$ and $\nabla_{\weight}\elag(\zv)\le\mbf{0}$, for all $\zv\in\mbf{Z}$. Moreover, $\mbf{g}_{\thv}(\zv,\widehat{\xi})\geq \mbf{0}$ and $\mbf{g}_{\weight}(\zv,\bar{\xi})\le \mbf{0}$ pointwise.
   	\end{lemma}
    \begin{proof}
   		By Lemma~\ref{lemma:bounds_operators}, we have that $\cw-\op^*\ul+\cbg\cdot\mbf{1}\geq\mbf 0$. Since in addition all the entries of $\phim$ are nonnegative, we conclude that $\nabla_{\thv}\elag(\zv)\geq\mbf{0}$. Moreover, $\frac{\partial}{\partial w_i}\elag(\zv)=\innerprod{\cost_i}{\mth}-\rho_{\cost_i}(\widehat{\expert})-2\le\norm{\cost_i}_\infty\norm{\mth}_1+1-2\le0$, for each $i\in[n_c]$, where we used H\"{o}lder's inequality and Lemma~\ref{lemma:bounds}. Similarly, we can show that $\mbf{g}_{\thv}(\zv,\widehat{\xi})\geq \mbf{0}$ and $\mbf{g}_{\weight}(\zv,\bar{\xi})\le \mbf{0}$.
   	\end{proof}

   	The following lemma shows that $\gv=
   	\begin{bmatrix}
   	\gth^\intercal	 &
   	-\gl^\intercal		 &
   	-\gw^\intercal		
   	\end{bmatrix}^\intercal
   	$ is an ubiased estimator of the monotone operator $\mbf{G}(\cdot)$ and establishes useful bounds for the error analysis of Algorithm~\ref{alg:main}.
   	\begin{lemma}\label{lemma:gradient:bounds}
   		Let $\Xi\triangleq\widehat{\Xi}\times\tilde{\Xi}\times\bar{\Xi}$ and let $\gv:\mbf{Z}\times\Xi\rightarrow\Re^{n_\mu+n_u+n_c}$ given by $\gv(\zv,\xi)\triangleq
   		\begin{bmatrix}
   		\gth(\zv,\hat{\xi})^\intercal	 &
   		-\gl(\zv,\tilde{\xi})^\intercal		 &
   		-\gw(\zv,\bar{\xi})^\intercal		
   		\end{bmatrix}^\intercal
   		$ as in Definition~\ref{def:stochastic_gradients}. The following statements are true.
   		\begin{enumerate}
   			\item \emph{\textbf{[unbiased gradient estimates]}} It holds that $\Exp_{\hat{\xi}}\left[\mbf{g}_{\thv}(\zv,\hat{\xi})\right]=\nabla_{\thv}\elag(\zv)$, $\Exp_{\tilde{\xi}}\left[\mbf{g}_{\lv}(\zv,\tilde{\xi})\right]=\nabla_{\lv}\elag(\zv)$, and $\Exp_{\bar{\xi}}\left[\mbf{g}_{\weight}(\zv,\bar{\xi})\right]=\nabla_{\weight}\elag(\zv)$, for all $\zv\in\mbf{Z}$.
   			\item \textbf{\emph{[$\boldsymbol{\norm{\cdot}_\infty}$-norm bounds]}} It holds that $\norm{\gth(\zv,\hat{\xi})}_\infty\le\frac{4\beta n_\mu}{1-\gamma}$ , $\norm{\gl(\zv,\tilde{\xi})}_\infty\le\frac{4\beta }{1-\gamma}$, and $\norm{\gw(\zv,\bar{\xi})}_\infty\le 4$, for all $\zv\in\mbf{Z}$, and for all $\xi\in\Xi$. 
   			\item \textbf{\emph{[dual and local-norm bounds]}} It holds that $\norm{\gth(\zv,\hat{\xi})}_{\thv'}^2\le\frac{16\beta^2 n_\mu^2}{(1-\gamma)^2}$, $\norm{\gl(\zv,\tilde{\xi})}_{2}^2\le\frac{4 n_u}{(1-\gamma)^2}$, and $\norm{\gw(\zv,\bar{\xi})}_{\weight'}^2\le 16$, for all $\zv\in\mbf{Z}$, for all $\xi\in\Xi$, for all $\thv'\in\Theta$, and for all $\weight'\in\mathcal{W}$.  
   			\item \textbf{\emph{[dual and local-norm second-moment bounds]}} It holds that $\Exp_{\hat{\xi}}\left[\norm{\gth(\zv,\hat{\xi})}_{\thv'}^2\right]\le\frac{16\beta^2 n_\mu}{(1-\gamma)^2}$, $\Exp_{\tilde{\xi}}\left[\norm{\gl(\zv,\tilde{\xi})}_{2}^2\right]\le\frac{4 n_u}{(1-\gamma)^2}$, and $\Exp_{\bar{\xi}}\left[\norm{\gw(\zv,\bar{\xi})}_{\weight'}^2\right]\le 16$, for all $\zv\in\mbf{Z}$, for all $\thv'\in\Theta$, and for all $\weight'\in\mathcal{W}$.
   			
   		\end{enumerate}
   		
   	\end{lemma}
   	\begin{proof}
   		\begin{enumerate}
   			\item For all $i\in[n_\mu]$, we have that
   		\begin{align*}
   		\Exp_{\hat{\xi}}\left[\mbf{g}_{\thv,i}(\zv,\hat{\xi})\right]=&\sum_{\hi=1}^{n_\mu}\sum_{\hx,\ha,\hy}\frac{1}{n_\mu}\phiv_{\hi}(\hx,\ha)P(\hy|\hx,\ha)\mbf{g}_{\thv,i}(\zv,\hat{\xi})\\
   			=& \sum_{\hx,\ha,\hy} \frac{1}{n_\mu}\phiv_{i}(\hx,\ha)P(\hy|\hx,\ha)	\frac{n_{\mu}\left( (1-\gamma)\cost_{(\hx,\ha)}^\intercal\weight-(\mbs{\psi}_{\hx}-\gamma\mbs{\psi}_{\hy})^\intercal\lv+2\beta\right)}{1-\gamma}\\
   				=&\innerprod{\phiv_i}{\cw-\op^*\ul+\cbg\cdot\mbf{1}}=\frac{\partial}{\partial\theta_i}\elag(\zv).
   				\end{align*}
   			Moreover, we have that
   			\begin{align*}
   			\Exp_{\tilde{\xi}}\left[\gl(\zv,\tilde{\xi})\right]=& \sum_{i=1}^{n_\mu}\sum_{x,a,y,x'}\theta_i\phiv_i(x,a)P(y|x,a)\nu_0(x')\left(\psiv_{x'}-\frac{\psiv_x-\gamma\psiv_y}{1-\gamma}\right)\\
   			=&\underbrace{ \sum_{x'} \nu_0(x')\psiv_{x'}}_{=\psim^\intercal\initial}-\frac{1}{1-\gamma}(\underbrace{\sum_{x,a}\mu_{\thv}(x,a)\psiv_x}_{=\psim^\intercal\bmat^\intercal\mth}-\gamma\underbrace{\sum_{x,a}\mu_{\thv}(x,a)\sum_y P(y|x,a)\psiv_y}_{=\psim^\intercal\pmat^\intercal\mth})\\
   			=& \psim^\intercal\left(\initial-\op\mth\right)=\nabla_{\lv}\elag(\zv).
   			\end{align*}
   			Finally, we have that
   				\begin{align*}
   				\Exp_{\bar{\xi}}\left[\gw(\zv,\bar{\xi})\right]=\sum_{i=1}^{n_\mu}\sum_{x,a}\theta_i\phiv_i(x,a)\cost_{(x,a)}-\efev-2\cdot\mbf{1}
   				= \mbf{C}^\intercal\mth-\efev-2\cdot\mbf{1}=\nabla_{\weight}\elag(\zv).
   				\end{align*}
   			
   		\item It holds that 
   		\begin{align*}
   		\norm{\gth(\zv,\hat{\xi})}_\infty=\left|n_{\mu}\left(\cw(\hx,\ha)-\frac{\ul(\hx)-\gamma\ul(\hy)-2\beta}{1-\gamma}\right)\right|\le\frac{4\beta n_\mu}{1-\gamma},
   		\end{align*}
   		where we used the triangle inequality and the bounds in Lemma~\ref{lemma:bounds}. Moreover, it holds that
   		\begin{align*}
   		\norm{\gl(\zv,\tilde{\xi})}_\infty=\norm{\psiv_{\txp}-\frac{\psiv_{\tx}-\gamma\psiv_{\ty}}{1-\gamma}}_\infty\le\frac{2}{1-\gamma},
   		\end{align*}
   		where we used the triangle inequality and that $\norm{\psiv_x}_\infty\le 1$, for all $x\in\sspace$. Finally, it holds that
   			\begin{align*}
   			\norm{\gw(\zv,\bar{\xi})}_\infty=\norm{\cost_{(\bx,\ba)}-\efev-2\cdot\mbf{1}}_{\infty}\le 4,
   			\end{align*}
   			where we used the triangle inequality, that $\norm{\cost_{(\bx,\ba)}}_\infty\le 1$, and that $\norm{\efev}_\infty\le 1$.
   			\item Since $\thv'\in\Theta=\Delta_{[n_\mu]}$, we have that
   			\[
   			\norm{\gth(\zv,\hat{\xi})}_{\thv'}^2\le	\norm{\gth(\zv,\hat{\xi})}_{\infty}^2\le\frac{16\beta^2 n_\mu^2}{(1-\gamma)^2}.
   			\]
   			Moreover,
   			\[
   			\norm{\gl(\zv,\tilde{\xi})}_2^2=\norm{\psiv_{\txp}-\frac{\psiv_{\tx}-\gamma\psiv_{\ty}}{1-\gamma}}_2^2\le\frac{4n_u}{(1-\gamma)^2},
   			\]
   			where we used the triangle inequality and that $\norm{\psiv_x}_2\le \sqrt{n_u}$, for all $x\in\sspace$. Finally, since $\weight'\in\mathcal{W}=\Delta_{[n_c]}$, we have
   			\[
   			\norm{\gw(\zv,\bar{\xi})}_{\weight'}^2\le\norm{\gw(\zv,\bar{\xi})}_{\infty}^2\le 16.
   			\]
   			\item We need only to prove the first inequality. We have,
   			\begin{align*}
   			\Exp_{\hat{\xi}}\left[\norm{\gth(\zv,\hat{\xi})}_{\thv'}^2\right]=&\Exp_{\hat{\xi}}\left[\gv^2_{\thv,\hi}(\zv,\hat{\xi})\theta'_{\hi}\right]\\
   			=&\sum_{\hi=1}^{n_\mu}\sum_{\hx,\ha,\hy}\frac{1}{n_\mu}\phiv_{\hi}(\hx,\ha)P(\hy|\hx,\ha)\mbf{g}^2_{\thv,\hi}(\zv,\hat{\xi})\theta^\prime_{\hi}\\
   			\le&\frac{1}{n_\mu}\norm{\gth(\zv,\hat{\xi})}^2_\infty\sum_{\hi=1}^n\theta_{\hi}^\prime\le\frac{16\beta^2 n_\mu}{(1-\gamma)^2}.
   			\end{align*}

   		\end{enumerate}
   		
   	\end{proof}

    \section{Proof of Theorem~\ref{th:strong}} \label{D}
    We are now going to prove Theorem~\ref{th:strong}. Our optimization variable is $\mbf{z}=(\thv,\lv,\weight)\in\mbf{Z}$, where $\mbf{Z}\triangleq\Theta\times\Lambda\times\mathcal{W}$ is the decision space. The update rule of stochastic primal-dual mirror descent in Algorithm~\ref{alg:main} is
    	\begin{equation}\label{eq:update-rule}
    	\mbf{z}_{n+1}=\arg\min_{\mbf{z}\in\mbf{Z}} \big(\innerprod{\gv(\zv_n,\xi_n)}{\mbf{z}}+\tfrac{1}{\eta}B_{R}(\mbf{z}||\mbf{z}_n)\big),
    	\end{equation}
    	where $\gv(\zv,\xi)\triangleq
    	\begin{bmatrix}
    	\gth(\zv,\hat{\xi})^\intercal	 &
    	-\gl(\zv,\tilde{\xi})^\intercal		 &
    	-\gw(\zv,\bar{\xi})^\intercal		
    	\end{bmatrix}^\intercal
    	$ is given by Definition~\ref{def:stochastic_gradients}, and $B_R$ is the following \emph{Bregman divergence}
    	\begin{equation}\label{eq:Bregman-divergence}
    	B_R(\zv||\zv')\triangleq \textup{KL}(\thv||\thv')+\frac{n_u}{2\beta^2}\norm{\lv-\lv'}_2^2+ \textup{KL}(\weight||\weight').
    	\end{equation}
    %\subsection{Proof of Theorem~\ref{th:strong}}
   
    \textbf{Theorem} $\mbf{1.}$ \emph{Let $\widehat{\pi}_N$ be the output of running Algorithm~\ref{alg:main} for $N=\max\left\{\mathcal{O}\left(\frac{\beta^2 n_\mu \log\left(\frac{1}{\delta}\right)}{(1-\gamma)^2\varepsilon^2}\right),\mathcal{O}\left(\frac{\beta\sqrt{n_\mu^3 \log\left(\frac{1}{\delta}\right)}}{(1-\gamma)\varepsilon}\right)\right\}$ iterations, with $m=\frac{8\log(\frac{4n_c}{\delta})}{\varepsilon^2}$ expert trajectories of length $H=\frac{1}{1-\gamma}\log(\frac{2}{\varepsilon})$, and learning rate $\eta=\frac{1-\gamma}{\beta\sqrt{N n_\mu}}$ . Then, with probability $1-\delta$ it holds that
    	$
    	\rho_{\true}(\widehat{\pi}_N)-\rho_{\true}(\expert)\le\varepsilon+\varepsilon_{\textup{approx},\widehat{\thv}_N}+\alpha^\star.
    	$
    If in addition Assumption~\ref{ass:coherence} is satisfied,  then with probability $1-\delta$, it holds that
    $
    \rho_{\true}(\widehat{\pi}_N)-\rho_{\true}(\expert)\le\varepsilon+\varepsilon_{\textup{approx},w}+\alpha^\star.
    $}
    \bigskip
   \newcommand{\vpn}{\val_{\weight_N^\star}^{\pN}}
   
   	Note that because the definition of the total expected cost is scaled by a factor $(1-\gamma)$, we finally have sample complexity $\max\left\{\mathcal{O}\left(\frac{\beta^2 n_\mu \log\left(\frac{1}{\delta}\right)}{(1-\gamma)^4\varepsilon^2}\right),\mathcal{O}\left(\frac{\beta\sqrt{n_\mu^3 \log\left(\frac{1}{\delta}\right)}}{(1-\gamma)^2\varepsilon}\right)\right\}$ for the conventional discounted MDP setting. Similarly, the expert sample complexity is $m=\frac{8\log(\frac{4n_c}{\delta})}{(1-\gamma)^2\varepsilon^2}$. The sample complexity scales linearly with the number of features  and does not depend on the number of states and actions.
    
    \begin{proofof}{Proof of Theorem~\ref{th:strong}} By Proposition~\ref{prop:strong} applied to $\widehat{\zv}_N=(\widehat{\thv}_N,\widehat{\lv}_N,\widehat{\weight}_N)$ and $\widehat{\pi}_N\triangleq\pi_{\mv_{\widehat{\thv}_N}}$ we get  
    	$$
    	\rho_{\true}(\pN)-\rho_{\true}(\expert)\le\sadav+\varepsilon_{\textup{approx},\thN}+\alpha^\star.
    	$$
    		where the approximation error $\varepsilon_{\textup{approx},\thN}$ is no larger than
    	$
    	\frac{2}{1-\gamma}\big(\beta\min_{\thv\in\Theta}\norm{\mth-\mv_{\apprentice}}_1+\min_{\lv\in\Lambda}\norm{\ul-\val_{\weight_N^\star}^{\pN}}_\infty\big),
    	$
    		and 
    		$
    		\weight^\star_N\triangleq \arg\max_{\weight\in\mathcal{W}}\left[\rho_{\weight}(\pN)-\rho_{\weight}(\expert)\right].
    		$
    		
    		Moreover, under Assumption~\ref{ass:coherence}, it holds that
    		\[
    		\rho_{\true}(\pth)-\rho_{\true}(\expert)\le3\epsilon_{\textup{sad}}(\boldsymbol{\theta},\boldsymbol{\lambda},\weight)+\varepsilon_{\textup{approx},w}+\alpha^\star.
    		\]
    	
    		So, in both cases we need to determine how many samples $N$ and expert trajectories $m$ are needed so that $\sadav\le\varepsilon$ with probability at least $1-\delta$. We have,
    		\begin{align}
    			\sadav=&\max_{(\lv,\weight)\in\Lambda\times\mathcal{W}}\rlag(\thN,\lv,\weight)-\min_{\thv\in\Theta}\rlag(\thv,\lN,\wN) \\
    		=&\max_{(\lv,\weight)\in\Lambda\times\mathcal{W}}\mlag(\thN,\lv,\weight)-\min_{\thv\in\Theta}\mlag(\thv,\lN,\wN)  \\
    		=&\max_{(\thv,\lv,\weight)\in\mbf{Z}}\left[\elag(\thN,\lv,\weight)-\elag(\thv,\lN,\wN)+\innerprod{\efev-\fev}{\weight-\wN}\right] \\
    		\le&\max_{(\thv,\lv,\weight)\in\mbf{Z}}\left[\elag(\thN,\lv,\weight)-\elag(\thv,\lN,\wN)+\norm{\efev-\fev}_\infty\norm{\weight-\wN}_1\right] \label{eq:almost-last}\\
    		\le&\max_{(\thv,\lv,\weight)\in\mbf{Z}}\left[\elag(\thN,\lv,\weight)-\elag(\thv,\lN,\wN)\right]+2\norm{\efev-\fev}_\infty, \label{eq:last}
    		\end{align}
    		where we used  H\"{o}lder's inequality in~(\ref{eq:almost-last}) and the triangle inequality in~(\ref{eq:last}). 
    		
    		By virtue of Lemma~1 in~\cite{Syed:2007} we have that for $
    		\norm{\efev-\fev}_\infty\le\frac{\varepsilon}{4}
    		$
    		 to hold with probability at least $1-\frac{\delta}{2}$, it suffices that $m\geq\frac{8\log(\frac{4n_c}{\delta})}{\varepsilon^2}$ and $H\geq\frac{1}{1-\gamma}\log(\frac{2}{\varepsilon})$.
    		Note that  Lemma~1 in~\cite{Syed:2007} is a variant of the Hoeffding's inequality (Theorem 2.5 in~\cite{McDiarmid:1998}), since we have to take into account that $\efev$ is a biased estimate of $\fev$ because the trajectories are truncated. By combining this last bound with~(\ref{eq:last}), we get that if $m\geq\frac{8\log(\frac{4n_c}{\delta})}{\varepsilon^2}$ and $H\geq\frac{1}{1-\gamma}\log(\frac{2}{\varepsilon})$, then with probability at least $1-\frac{\delta}{2}$ it holds that
    			\begin{equation}\label{eq:bound:wrt:elag}
    				\sadav\le\widehat{\epsilon}_{\textup{sad}}(\zN)+\frac{\varepsilon}{2},
    			\end{equation}
    	where $$\widehat{\epsilon}_{\textup{sad}}(\zN)\triangleq\max\limits_{(\lv,\weight)\in\Lambda\times\mathcal{W}}\elag(\thN,\lv,\weight)-\min\limits_{\thv\in\Theta}\elag(\thv,\lN,\wN)$$ is the empirical saddle-point residual for $\zN$. All in all, we need $\widehat{\epsilon}_{\textup{sad}}(\zN)\le\frac{\varepsilon}{2}$ with confidence $1-\frac{\delta}{2}$. 
    	
    	By linearity of $\elag(\cdot,\lv,\weight)$ we have that for all $n\in\mathbb{N}$, and for all $\thv\in\Theta$,
    	\begin{equation}\label{eq:convexity}
    	\elag(\thv_n,\lv_n,\weight_n)-\elag(\thv,\lv_n,\weight_n)=\innerprod{\nabla_{\thv}\elag(\thv_n,\lv_n,\weight_n)}{\thv_n-\thv}.
    	\end{equation}
    	Similarly, by linearity of $\elag(\thv,\cdot,\cdot)$ we have that for all $n\in\mathbb{N}$, and for all $(\lv,\weight)\in\Lambda\times\mathcal{W}$,
    	\begin{equation}\label{eq:bound:concavity}
    	\elag(\thv_n,\lv,\weight)-\elag(\thv_n,\lv_n,\weight_n)=\innerprod{\nabla_{\lv}\elag(\thv_n,\lv_n.\weight_n)}{\lv-\lv_n}+\innerprod{\nabla_{\weight}\elag(\thv_n,\lv_n,\weight_n)}{\weight-\weight_n}.
    		\end{equation}
   Therefore,
    \begin{align*}
    \widehat{\epsilon}_{\textup{sad}}(\zN)=&\frac{1}{N}\left(\max_{(\lv,\weight)\in\Lambda\times\mathcal{W}}\sum_{n=1}^N\elag(\thv_n,\lv,\weight)-\min_{\thv\in\Theta}\sum_{n=1}^n\elag(\thv,\lv_n,\weight_n)\right)              
    =\frac{1}{N}\max_{\zv\in\mbf{Z}}\sum_{n=1}^N\innerprod{\mbf{G}(\zv_n)}{\zv_n-\zv}, 
   % =&\frac{1}{N}(\underbrace{\athroisma{n}{1}{N}{\innerprod{\mbf{G}(\zv_n)-\gv_n}{\zv_n}}}_{\triangleq T_1}+\underbrace{\athroisma{n}{1}{N}{\innerprod{\gv_n}{\zv_n-\zv_N^\star}}}_{\triangleq T_2}+\underbrace{\athroisma{n}{1}{N}
   % {\innerprod{\gv_n-\mbf{G}(\zv_n)}{\zv^\star_N}}}_{\triangleq T_3})
    \end{align*}
    where the first equality follows once more by linearity of $\elag(\cdot,\lv,\weight)$ and  $\elag(\thv,\cdot,\cdot)$,  and the second equality by (\ref{eq:convexity}) and~(\ref{eq:bound:concavity}).
    
    \newcommand{\bthv}{\bar{\thv}}
    \newcommand{\blv}{\bar{\lv}}
    \newcommand{\bweight}{\bar{\weight}}
    \newcommand{\bzv}{\bar{\zv}}
    \newcommand{\cthv}{\mathring{\thv}}
    \newcommand{\clv}{\mathring{\lv}}
    \newcommand{\cweight}{\mathring{\weight}}
    \newcommand{\czv}{\mathring{\zv}}
    \newcommand{\zetav}{\boldsymbol{\zeta}}

    Consider now the sequence of \emph{ghost iterates} $(\bzv_n)_{n=1}^N$ in $\mbf{Z}$ generated by taking gradient step with $\mbf{G}(\zv_n)-\gv(\zv_n,\xi_n)$
    coupled with each iteration, i.e.,
    \[
        \bzv_1=\zv_1,\quad	\bzv_{n+1}=\arg\min_{\bzv\in\mbf{Z}} \big(\innerprod{\mbf{G}(\zv_n)-\gv(\zv_n,\xi_n)}{\bzv}+\tfrac{1}{\eta}B_{R}(\bzv||\bzv_n)\big),
    \]
    where $B_R$ is the same Bregman divergence as in (\ref{eq:Bregman-divergence}).
    
    We then have
    \begin{equation}\label{eq:decomposition}
    \max_{\zv\in\mbf{Z}}\sum_{n=1}^N\innerprod{\mbf{G}(\zv_n)}{\zv_n-\zv}\le\underbrace{\athroisma{n}{1}{N}{\innerprod{\Delta_n}{\zv_n-\bzv_n}}}_{\triangleq T_1}+\underbrace{\max_{\zv\in\mbf{Z}}\athroisma{n}{1}{N}{\innerprod{\gv_n}{\zv_n-\zv}}}_{\triangleq T_2}+\underbrace{\max_{\bzv\in\mbf{Z}}\athroisma{n}{1}{N}{\innerprod{\Delta_n}{\bzv_n-\bzv}}}_{\triangleq T_3},
    \end{equation}
    where 
    \[
    \gv_n\triangleq \gv(\zv_n,\xi_n)\quad \text{and}\quad \Delta_n\triangleq \mbf{G}(\zv_n)-\gv(\zv_n,\xi_n).
    \]
     Moreover, we set for brevity $\gv_{n,\thv}\triangleq\gv_{\thv}(\zv_n,\hat{\xi}_n)$, $\gv_{n,\lv}\triangleq\gv_{\lv}(\zv_n,\tilde{\xi}_n)$, $\gv_{n,\weight}\triangleq\gv_{\weight}(\zv_n,\bar{\xi}_n)$, $\Delta_{n,\thv}\triangleq\nabla_{\thv}\elag(\zv_n)-\gv_{\thv}(\zv_n,\hat{\xi}_n)$, $\Delta_{n,\lv}\triangleq\nabla_{\lv}\elag(\zv_n)-\gv_{\lv}(\zv_n,\tilde{\xi}_n)$, and $\Delta_{n,\weight}\triangleq\nabla_{\weight}\elag(\zv_n)-\gv_{\weight}(\zv_n,\bar{\xi}_n)$.                          
    
    We are now going to bound each term in~(\ref{eq:decomposition}) separately.
    
    \textbf{Bounding the term $\boldsymbol{T_1}$.} Let $\xi_n=(\hat{\xi}_n,\tilde{\xi}_n,\bar{\xi}_n)$ be the generated sample at round $n$ and let $\mathcal{F}_n=\sigma(\xi_1,\ldots,\xi_n)$ be the $\sigma$-algebra generated by the history process. The iterates $\zv_n=\zv_n(\xi_1,\ldots,\xi_{n-1})$ and $\bzv_n=\bzv_n(\xi_1,\ldots,\xi_{n-1})$ are functions of $\xi_1,\ldots,\xi_{n-1}$ (and thus they are random). Since in addition, $\zv_n$ and $\bzv_n$ are independent of the randomness of $\xi_n$ (we sample $\xi_n$ after the definition of $\zv_n$ and $\bzv_n$), we get by Lemma~\ref{lemma:gradient:bounds} that the sequence of random variables $Y_n=\innerprod{\Delta_n}{\zv_n-\bzv_n}$ forms a \emph{martingale difference}, i.e., $Y_n$ is $\mathcal{F}_n$-measurable and $\Exp\left[Y_n\mid\fn\right]=0$. In order to bound the term $T_1=\athroisma{n}{1}{N}{Y_n}$, we will use a Bernstein martingale concentration inequality (Theorem 3.15 in~\cite{McDiarmid:1998}). In particular, for any $\varepsilon\geq 0$,
    \begin{equation}\label{eq:Bernstein}
    \Prob\left[\athroisma{n}{1}{N}{Y_n}\geq\varepsilon\right]\le\exp\left(\frac{-\varepsilon^2}{2N\sigma^2(1+\frac{b\varepsilon}{3N\sigma^2})}\right),
    \end{equation}
    where $|Y_n|\le b$, for all $n\in\mathbb{N}$ and $\Var\left[ Y_n\mid\fn\right]\le\sigma^2$, for all $n\in\mathbb{N}$. We think of the term $\frac{b\varepsilon}{3N\sigma^2}$ as a negligible error term. In order to use~(\ref{eq:Bernstein}) we need to compute the constants $b$ and $\sigma^2$.
    We have
    \[
    Y_n=\innerprod{\Delta_n}{\zv_n-\bzv_n}=\innerprod{\Delta_{n,\thv}}{\thv_n-\bthv_n}-\innerprod{\Delta_{n,\lv}}{\lv_n-\blv_n}-\innerprod{\Delta_{n,\weight}}{\weight_n-\bweight_n}.
    \]
    For the term $\innerprod{\Delta_{n,\thv}}{\thv_n-\bthv_n}$ we have
    \begin{align}
    \left|\innerprod{\Delta_{n,\thv}}{\thv_n-\bthv_n}\right|\le&\left|\innerprod{\cost_{\weight_n}-\op^*\uv_{\lv_n}+C_{\beta,\gamma}\cdot\mbf{1}}{\phim\thv_n-\phim\bthv_n}\right|+\left|\innerprod{\gv_{n,\thv}}{\thv_n-\bthv_n}\right| \tag{By triangle inequality}\\
    \le& \norm{\cost_{\weight_n}-\op^*\uv_{\lv_n}+C_{\beta,\gamma}\cdot\mbf{1}}_\infty\norm{\mv_{\thv_n}-\mv_{\bthv_n}}_1 +\norm{\gv_{n,\thv}}_\infty\norm{\thv_n-\bthv_n}_1 \tag{By H\"{o}lder's inequality}\\
    \le& \frac{8\beta}{1-\gamma}+\frac{8\beta n_\mu}{1-\gamma}=\frac{8\beta(n_\mu+1)}{1-\gamma}. \tag{By Lemmas~\ref{lemma:bounds}, \ref{lemma:bounds_operators} and~\ref{lemma:gradient:bounds}}
    \end{align}
    Moreover, for the term $\innerprod{-\Delta_{n,\lv}}{\lv_n-\blv_n}$, we have
    \begin{align}
    \left|\innerprod{-\Delta_{n,\lv}}{\lv_n-\blv_n}\right|\le& \left|\innerprod{\initial-\op\mv_{\thv_n}}{\psim\lv_n-\psim\blv_n}\right|+\left|\innerprod{\gv_{n,\lv}}{\lv_n-\blv_n}\right| \tag{By triangle inequality}\\
    \le& \norm{\initial-\op\mv_{\thv_n}}_1\norm{\uv_{\lv_n}-\uv_{\blv_n}}_\infty+\norm{\gv_{n,\lv}}_\infty\norm{\lv_n-\blv_n}_1 \tag{By H\"{o}lder's inequality}\\
    \le& \frac{8\beta}{1-\gamma}. \tag{By Lemmas~\ref{lemma:bounds}, \ref{lemma:bounds_operators} and~\ref{lemma:gradient:bounds}}
    \end{align}
     Finally, for the term $\innerprod{-\Delta_{n,\weight}}{\weight_n-\bweight_n}$, we have
     \begin{align}
     \left|\innerprod{-\Delta_{n,\weight}}{\weight_n-\bweight_n}\right|=& \left|\innerprod{\cost_{(\bar{x},\bar{a})}}{\weight_n-\bweight_n}+\innerprod{\mv_{\thv_n}}{\cma\weight_n-\cma\bweight_n}\right|  \tag{By Definition~\ref{def:stochastic_gradients}}  \\
     \le& \norm{\cost_{(\bar{x},\bar{a})}}_\infty\norm{\weight_n-\bweight_n}_1+\norm{\mv_{\thv_n}}_1\norm{\cost_{\weight_n}-\cost_{\bweight_n}}_\infty \tag{By H\"{o}lder's inequality}\\
     \le& 4. \tag{By Lemma~\ref{lemma:bounds}}
     \end{align}
     All in all, $|Y_n|\le b\triangleq\frac{8\beta(n_\mu+2)}{1-\gamma}+4$.
     
     We are now going to bound the conditional variance of $Y_n$. We have
      \begin{align}
      \Var\left[\innerprod{\Delta_{n,\thv}}{\thv_n-\bthv_n}\mid\fn\right]=&\Exp\left[\innerprod{\Delta_{n,\thv}}{\thv_n-\bthv_n}^2\mid\fn\right] \label{first} \\
       =& \Exp\left[\innerprod{\gv_{n,\thv}}{\thv_n-\bthv_n}^2\mid\fn \right] - \innerprod{\nabla_{\thv}\elag(\zv_n)}{\thv_n-\bthv_n}^2 \label{second} \\
      \le& \Exp\left[\innerprod{\gv_{n,\thv}}{\thv_n-\bthv_n}^2\mid\fn \right]\\
      =& \Exp\left[\gv_{n,\thv,\hi_n}^2(\theta_{n,\hi_n}-\bar{\theta}_{n,\hi_n})^2\mid\fn\right] \label{fifth}\\
      =& \sum_{i=1}^{n_\mu}\sum_{x,a,y}\frac{1}{n_\mu}\phi_i(x,a) P(y|x,a) \gv_{n,\thv,i}^2(\theta_{n,i}-\bar{\theta}_{n,i})^2 \label{sixth}\\
      \le&  \frac{1}{n_\mu} \norm{\gv_{n,\thv}}_\infty^2 \norm{\thv_n-\bthv_n}_2^2 \label{seventh}\\
       \le&  \frac{1}{n_\mu} \norm{\gv_{n,\thv}}_\infty^2 \norm{\thv_n-\bthv_n}_1^2 \label{eighth}\\
      \le&\frac{64\beta^2n_\mu}{(1-\gamma)^2} \label{last},
      \end{align}
      where in~(\ref{first}) we used that $\Exp[\innerprod{\Delta_{n,\thv}}{\thv_n-\bthv_n}\mid\fn]=0$, in~(\ref{second}) we used that the random variables $\zv_n-\bzv_n$ and $\thv_n-\bthv_n$ are $\fn$-measurable and that $\Exp\left[\gv_{n,\thv}\mid\fn\right]=\nabla_{\thv}\elag(\zv_n)$, in~(\ref{sixth}) and~(\ref{seventh}) we used Definition~\ref{def:stochastic_gradients}, in~(\ref{seventh}) we used that $\phiv_i\in\Delta_{\sspace\times\aspace}$ and $P(\cdot|x,a)\in\Delta_{\sspace}$, in~(\ref{eighth}) we used that $\norm{\thv}_2\le\norm{\thv}_1$, and in~(\ref{last}) we used the triangle inequality and Lemma~\ref{lemma:gradient:bounds}.
     
     By similar arguments, we have
      \begin{align}
      \Var[\innerprod{\Delta_{n,\lv}}{\lv_n-\blv_n}\mid\fn]\le \Exp\left[\innerprod{\gv_{n,\lv}}{\lv_n-\blv_n}^2\mid\fn\right]\le\norm{\gv_{n,\lv}}^2_\infty\norm{\lv_n-\blv_n}_1^2\le\frac{16\beta^2}{(1-\gamma^2)}.
      \end{align}
      Finally,
     \begin{align}
      \Var[\innerprod{\Delta_{n,\weight}}{\weight_n-\bweight_n}\mid\fn]\le \Exp\left[\innerprod{\gv_{n,\weight}}{\weight_n-\bweight}^2\mid\fn\right]\le\norm{\gv_{n,\weight}}^2_\infty\norm{\weight_n-\bweight_n}_1^2\le64.
    \end{align}
   Now, since $\hat{\xi}_n$, $\tilde{\xi}_n$, $\bar{\xi}_n$ are independent and both $\zv_n$ and $\bzv_n$ are $\fn$-measurable, we get that $\innerprod{\Delta_{n,\thv}}{\thv_n-\bthv_n}$, $\innerprod{\Delta_{n,\lv}}{\lv_n-\blv_n}$, and $\innerprod{\Delta_{n,\weight}}{\weight_n-\bweight_n}$ are independent given $\fn$. Therefore,
     \begin{align*}
     \Var[Y_n\mid\fn]=&\Var[\innerprod{\Delta_{n,\thv}}{\thv_n-\bthv_n}\mid\fn]+\Var[\innerprod{\Delta_{n,\lv}}{\lv_n-\blv_n}\mid\fn]+\Var[\innerprod{\Delta_{n,\weight}}{\weight_n-\bweight_n}\mid\fn] \\
     \le& \,\,\sigma^2\triangleq \frac{16\beta^2(4n_\mu+1)}{(1-\gamma)^2}+64. 
       \end{align*}
     Therefore by the Bernstein martingale concentration inequality~(\ref{eq:Bernstein}) we have with probability at least $1-\delta$
    \begin{equation}
    \athroisma{n}{1}{N}{Y_n}\le\sqrt{2N\sigma^2(1+o(1))\log\left(\frac{1}{\delta}\right)}=\mathcal{O}\left(\frac{\beta\sqrt{N n_\mu\log(\frac{1}{\delta})}}{1-\gamma}\right).
    \end{equation}
  
   \textbf{Bounding the term $\boldsymbol{T_2}$.} By Lemma~\ref{lemma:nonnegative_gradients}, we know that $\gv_{n,\thv}\geq\mbf{0}$ and $-\gv_{n,\weight}\le\mbf{0}$, for all $n\in\mathbb{N}$. Therefore, by Lemma~10 and Lemma~11 in~\cite{Jin:2020}\footnote{Note that Lemma~11 in~\cite{Jin:2020} still holds for $\boldsymbol{\gamma}_t\geq -\boldsymbol{1}$ since the inequality $e^{-a}\le 1-a+a^2$ holds for all $a>-1$} for a fixed $\zv\in\mbf{Z}$, we have that
   \begin{align}
   \athroisma{n}{1}{N}{\innerprod{\gv_n}{\zv_n-\zv}}\le&\frac{1}{\eta}B_R(\zv||\zv_1)+\athroisma{n}{1}{N}{\left[ \innerprod{\gv_n}{\zv_n-\zv_{n+1}}-\frac{1}{\eta}B_R(\zv_{n+1}||\zv_n) \right]}\\
   \le& \frac{1}{\eta}B_R(\zv||\zv_1)+\frac{\eta}{2}\athroisma{n}{1}{N}{\left[\norm{\gv_{n,\thv}}_{\thv_n}^2+\frac{\beta^2}{n_u}\norm{\gv_{n,\lv}}_2^2+\norm{\gv_{n,\weight}}_{\weight_n}^2 \right]}.
   \end{align}
   Now, we define
   \[
   Z_n= \left[\norm{\gv_{n,\thv}}_{\thv_n}^2+\frac{\beta^2}{n_u}\norm{\gv_{n,\lv}}_2^2+\norm{\gv_{n,\weight}}_{\weight_n}^2 \right] -\Exp\left[\norm{\gv_{n,\thv}}_{\thv_n}^2+\frac{\beta^2}{n_u}\norm{\gv_{n,\lv}}_2^2+\norm{\gv_{n,\weight}}_{\weight_n}^2\mid\fn \right].
   \]
   Clearly $Z_n$ is an $\mathcal{F}_n$-martingale difference. Therefore, by the Azuma-Hoeffding inequality~(Theorem 3.14 in~\cite{McDiarmid:1998}), with probability at least $1-\delta$, it holds that
   \begin{equation}
   	\athroisma{n}{1}{N}{Z_n}\le\sqrt{\frac{N\log\left(\frac{1}{\delta}\right)}{2}}\cdot M_1,
   \end{equation}
   where $|Z_n|\le M_1$, for all $n\in\mathbb{N}$.
   
   Putting them all together we  get that with probability at least $1-\delta$
   \begin{equation}\label{eq:putting-together}
   \max_{\zv\in\mbf{Z}}\athroisma{n}{1}{N}{\innerprod{\gv_n}{\zv_n-\zv}}\le\frac{1}{\eta}\underbrace{\max_{\zv\in\mbf{Z}}B_R(\zv||\zv_1)}_{\le M_2}+\frac{\eta}{2}\underbrace{\athroisma{n}{1}{N}{\Exp\left[\norm{\gv_{n,\thv}}_{\thv_n}^2+\frac{\beta^2}{n_u}\norm{\gv_{n,\lv}}_2^2+\norm{\gv_{n,\weight}}_{\weight_n}^2\mid\fn \right]}}_{\le M_3}+\frac{\eta M_1}{2}\sqrt{\frac{N\log\left(\frac{1}{\delta}\right)}{2}}.
   \end{equation}
   By using the gradient bounds in Lemma~\ref{lemma:gradient:bounds}, we get 
   \[
   M_3\triangleq N\left(\frac{16\beta^2n_\mu}{(1-\gamma)^2}+\frac{\beta^2}{n_u}\frac{4n_u}{(1-\gamma)^2}+16\right)=\mathcal{O}\left(\frac{\beta^2 n_\mu N}{(1-\gamma)^2}\right)
   \] 
   and
   \[
   M_1\triangleq 2 \left(\frac{16\beta^2n_\mu^2}{(1-\gamma)^2}+\frac{\beta^2}{n_u}\frac{4n_u}{(1-\gamma)^2}+16\right)=\mathcal{O}\left(\frac{\beta^2 n_\mu^2}{(1-\gamma)^2}\right).
   \]
   Moreover, by considering the Bregman divergence given in~(\ref{eq:Bregman-divergence}) and that $\zv_1=(\tfrac{1}{n_\mu}\mbf{1},\mbf{0},\tfrac{1}{n_c}\mbf{1})$, we get that
   \[
   M_2\triangleq \left(\log n_\mu+\log n_c+\tfrac{1}{2}\right).
   \]
   Therefore, by setting $\eta\triangleq\frac{1-\gamma}{\beta\sqrt{N n_\mu}}$ we get by~(\ref{eq:putting-together}) that with probability at least $1-\delta$
   \begin{equation}
   \max_{\zv\in\mbf{Z}}\athroisma{n}{1}{N}{\innerprod{\gv_n}{\zv_n-\zv}}\le\mathcal{O}\left(\frac{\beta\sqrt{N n_\mu}}{1-\gamma}\right)+\mathcal{O}\left(\frac{\beta\sqrt{N n_\mu}}{1-\gamma}\right)+\mathcal{O}\left(\frac{\beta\sqrt{n_\mu^3\log\left(\frac{1}{\delta}\right)}}{1-\gamma}\right).
   \end{equation}
  \textbf{Bounding the term $\boldsymbol{T_3}$.} By the choice of $\eta\triangleq\frac{1-\gamma}{\beta\sqrt{N n_\mu}}$ and Lemma~\ref{lemma:gradient:bounds} we get that for all $n$, it holds that $\norm{\eta\Delta_{n,\thv}}_\infty\le 1$ and $\norm{\eta\Delta_{n,\weight}}_\infty\le 1$. Indeed, note for example that by Jensen's inequality for conditional expectations and Lemma~\ref{lemma:gradient:bounds}, we have $\norm{\eta\Delta_{n,\thv}}_\infty\le 2\eta\norm{\mbf{g}_{n,\thv}}_\infty\le 8\sqrt{\frac{n_\mu}{N}}\le 1$, for sufficiently large $N$. Therefore, once more we can combine regret bounds of stochastic mirror descent with local norms (Lemma~10 and Lemma~11 in~\cite{Jin:2020}) and the Azuma-Hoeffding inequality (Theorem~3.14 in~\cite{McDiarmid:1998} ) to end up that with probability at least $1-\delta$
    \begin{equation}\label{eq:putting-together-2}
    \max_{\bzv\in\mbf{Z}}\athroisma{n}{1}{N}{\innerprod{\Delta_n}{\bzv_n-\bzv}}\le\frac{1}{\eta}\underbrace{\max_{\bzv\in\mbf{Z}}B_R(\bzv||\bzv_1)}_{\le m_2}+\frac{\eta}{2}\underbrace{\athroisma{n}{1}{N}{\Exp\left[\norm{\Delta_{n,\thv}}_{\bthv_n}^2+\frac{\beta^2}{n_u}\norm{\Delta_{n,\lv}}_2^2+\norm{\Delta_{n,\weight}}_{\bweight_n}^2\mid\fn \right]}}_{\le m_3}+\frac{\eta m_1}{2}\sqrt{\frac{N\log\left(\frac{1}{\delta}\right)}{2}},
    \end{equation}
    where
     \[
     \left[\norm{\Delta_{n,\thv}}_{\thv_n}^2+\frac{\beta^2}{n_u}\norm{\Delta_{n,\lv}}_2^2+\norm{\Delta_{n,\weight}}_{\weight_n}^2 \right] -\Exp\left[\norm{\Delta_{n,\thv}}_{\thv_n}^2+\frac{\beta^2}{n_u}\norm{\Delta_{n,\lv}}_2^2+\norm{\Delta_{n,\weight}}_{\weight_n}^2\mid\fn \right]\le m_1.
     \]
  By using similar arguments as the ones regarding the term $T_2$, we get that with probability at least $1-\delta$
  \begin{equation}
  \max_{\bzv\in\mbf{Z}}\athroisma{n}{1}{N}{\innerprod{\Delta_n}{\bzv_n-\bzv}}\le\mathcal{O}\left(\frac{\beta\sqrt{N n_\mu}}{1-\gamma}\right)+\mathcal{O}\left(\frac{\beta\sqrt{N n_\mu}}{1-\gamma}\right)+\mathcal{O}\left(\frac{\beta\sqrt{n_\mu^3\log\left(\frac{1}{\delta}\right)}}{1-\gamma}\right).
  \end{equation}
  
   \textbf{Putting them all together.} Combining the bounds for the three terms $T_1$, $T_2$ and $T_3$ above and using a union bound, we get that with probability at least $1-\delta$
   \begin{equation}
    \widehat{\epsilon}_{\textup{sad}}(\zN)\le\mathcal{O}\left(\frac{\beta\sqrt{n_\mu \log\left(\frac{1}{\delta}\right)}}{(1-\gamma)\sqrt{N}}\right)+\mathcal{O}\left(\frac{\beta\sqrt{n_\mu^3 \log\left(\frac{1}{\delta}\right)}}{(1-\gamma)N}\right).
   \end{equation}
   Thus the sample complexity of mirror descent is $N=\max\left\{\mathcal{O}\left(\frac{\beta^2 n_\mu \log\left(\frac{1}{\delta}\right)}{(1-\gamma)^2\varepsilon^2}\right),\mathcal{O}\left(\frac{\beta\sqrt{n_\mu^3 \log\left(\frac{1}{\delta}\right)}}{(1-\gamma)\varepsilon}\right)\right\}$.

     \end{proofof} 
     
     \section{Numerical Experiments}\label{app:numerics}\label{E}
     
     In this section we analyze performance of the policy extracted from a saddle point of the empirical Lagrangian
     \[
     \widehat{\mcf{L}}(\mv,\uv,\weight)\triangleq\innerprod{\mv}{\cost_\weight-\op^*\uv} - \innerprod{\weight}{\efev} + \innerprod{\initial}{\uv}
     \]
     as a function of the number of trajectories $m$ used to compute the empirical feature expectation vector $\efev$.
     The length of each trajectory is set to $H=200$.
     These preliminary results do not consider approximation of $\mv$ and $\uv$ by linear combinations of features.
     
     Our environment is a $16$-by-$16$ gridworld with state space $\sspace=\{(i,j) \mid i\in[16], \, j\in[16]\}$ and action space $\aspace=\{\texttt{`up'},\texttt{`down'},\texttt{`left'},\texttt{`right'}\}$.
     The initial state distribution is uniform over $\sspace$.
     When an action is played, the executed direction will be the same as the intended with probability $0.7$ and any of the remaining directions with probability $0.1$.
     If the executed direction does not take the agent outside the grid, then the agent will move in the intended direction; otherwise, it will remain at the same cell.
     
     We generate $5$ cost vectors as $\cost_i = \bar{\cost}_i \otimes \ones$ with $\ones\in\Re^{|\aspace|}$, where $\otimes$ denotes the Kronecker product and elements of $\bar{\cost}_i\in\Re^{|\sspace|}$ are drawn from the set $\{0,-1,1\}$ with probabilities $\{0.8, 0.1, 0.1\}$, respectively.
     The expert policy is an optimal one for cost vector $\true=\cost_1$.
     
     Figure~\ref{fig:numerics}
     %Figure~2
     shows statistics for the performance of the learned policies depending on the number of sample trajectories $m$ used to compute $\efev$.
     For each $m$ we performed $20$ runs.
     The figure shows that, even for a modest number of sample trajectories, the performance of the extracted policy is close to that of the expert.
     
     \begin{figure}
     	\centering
     	\begin{tikzpicture}
     	\begin{semilogxaxis} [
     	width  = .5\textwidth,
     	height = .4\textwidth,
     	minor tick style = {draw=none},
     	ylabel style = {yshift=.5em},
     	ymin = -.38,
     	ymax = -.15,
     	xlabel = {Number of sample trajectories $m$},
     	ylabel = {$\mbs{\rho}_{\true}(\hat{\pi})$},
     	]
     	% expert performance
     	\addplot [mark=none, red, dashed] coordinates {(1e0,-0.3536) (1e3,-0.3536)};
     	
     	% error bars
     	\addplot plot [blue, mark options={solid, blue,scale=.8}, error bars/.cd, y dir=both, y explicit] table [x=n_traj, y=rho_mean, y error plus=rho_std, y error minus=rho_std, col sep=comma] {data/sample_traj.csv};
     	
     	% shaded area
     	\addplot [name path=upper,draw=none] table[x=n_traj,y=rho_max,col sep=comma] {data/sample_traj.csv};
     	\addplot [name path=lower,draw=none] table[x=n_traj,y=rho_min,col sep=comma] {data/sample_traj.csv};
     	\addplot [fill=blue!10] fill between[of=upper and lower];
     	\end{semilogxaxis}
     	\end{tikzpicture}
     	\label{fig:numerics}
     	\caption{%
     		Performance of learned policies for a $16$-by-$16$ gridworld environment.
     		The dashed line shows $\mbs{\rho}_{\true}(\expert)$.
     		The error bars show the mean and standard deviation over $20$ runs and the shaded area provides bounds for the computed policies.
     	}
     \end{figure}
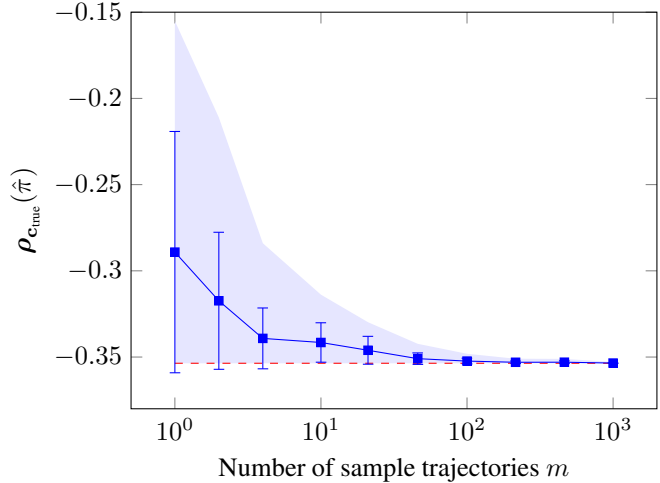

     	\section{A No-Regret Online Learning View}\label{F}
     	An online learning problem~\cite{Shalev:2012,Hazan:2016} describes the iterative interactions between a learner and an opponent. In round n, the learner chooses a \emph{decision} $\zv_n$ in the decision set $\mbf{Z}$, the opponent chooses a \emph{per-round loss function} $\ell_n:\mbf{Z}\rightarrow\ar$, and information about $\nabla \ell_n(\zv_n)$ is revealed to the learner.
     	
     	After $N$ rounds, the goal is to minimize the \emph{regret}
     	\[
     	\rm{Regret}_N(\mbf{Z})\triangleq \sum_{n=1}^N\ell_n(\zv_n)-\min_{z\in\mbf{Z}}\sum_{n=1}^N\ell_n(\zv).
     	\]
     	We argue that solving the online learning problem with $\mbf{Z}=\Theta\times\Lambda\times\mathcal{W}$ and per-round loss function
     	\begin{align*}
     	\ell_n(\zv)\triangleq&\rlag(\thv,\lv_n,\weight_n)-\rlag(\thv_n,\lv,\weight)\\
     	=& \innerprod{\mth-\mexp}{\cost_{\weight_n}-\op^*\uv_{\lv_n}}-\innerprod{\mv_{\thv_n}-\mexp}{\cw-\op^*\ul},
     	\end{align*}
     	is equivalent to solving the linearly-relaxed saddle-point problem~(\ref{LRALP}). 
     	
     	Note that the per-lound loss $\ell_n$ is linear in $\zv=(\thv,\lv,\weight)$. Moreover, the corresponding online learning problem is \emph{continuous} and \emph{predictable}. 
     	%Indeed, by considering the bifunction $(\zv,\zv')\mapsto f_{\zv}(\zv')\triangleq\rlag(\thv',\lv,\weight)-\rlag(\thv,\lv',\weight')$, we have that $\ell_n(\zv)=f_{\zv_n}(\zv)$ and the map $\zv\mapsto\nabla f_{\zv}(\zv')$ is continuous, for all $\zv'\in\mbf{Z}$. 
     	In other words, the per-round loss $\ell_n$ changes continuously with respect to the played decisions of the learner, and the environment is not completely adversarial since the transition matrix $\pmat$ and the expert policy $\expert$ remain the same across different rounds of online learning. A similar reduction of the forward RL problem to no-regret online learning has been investigated in~\cite{Cheng:2020}.
     	
     	To see why the two problems are equivalent, first of all note that $\nabla_{\zv}\ell_n(\zv_n)=
		\begin{bmatrix}
		\nabla_{\thv}\rlag(\zv)^\intercal		 &
		-\nabla_{\lv}\rlag(\zv)^\intercal		 &
		-\nabla_{\weight}\rlag(\zv)^\intercal	
		\end{bmatrix}^\intercal$. Moreover, by the bilinearity of the modified Lagrangian $\rlag$ we get that 
     		\begin{align*}
     		&{\epsilon}_{\text{sad}}(\thN,\lN,\wN)\\
     		=&\max_{\zv\in\mbf{Z}}\left[\rlag(\thN,\lv,\weight)-\rlag(\thv,\lN,\wN)\right]\\
     		=&\frac{1}{N} \max_{\zv\in\mbf{Z}}\sum_{i=1}^N\left[\rlag(\thv_n,\lv,\weight)-\rlag(\thv,\lv_n,\weight_n)\right]\\
     		=& \frac{1}{N}\left( \sum_{n=1}^N\ell_n(\zv_n)-\min_{z\in\mbf{Z}}\sum_{n=1}^N\ell_n(\zv)\right)=\frac{1}{N}\rm{Regret}_N(\mbf{Z}).
     		\end{align*}
     	
     	 Therefore, any stochastic first-order primal-dual algorithm for the bilinear Lagrangian $\rlag$ is equivalent to the corresponding online learning algorithm with first-order noisy feedback for the linear losses $\ell_n$. Moreover, the average $\zN$ is an $(\varepsilon,\delta)$-optimal saddle point to~(\ref{LRALP}) if and only if the average regret $\frac{1}{N}\rm{Regret}_N(\mbf{Z})<\varepsilon$ with probability $1-\delta$.  In particular, Algorithm~\ref{alg:main} can be seen as online mirror descent for the  empirical per-round loss function $\widehat{\ell}_n$.
     	
     	It is worth noting that this online learning approach differs from the one in~\cite{ross2011reduction} where interaction with the expert is required. It is also different from the game-theoretic approach in~\cite{Syed:2007} where the forward RL problem has to be solved repeatedly.
     
     \section{Future Directions}\label{G}
     There is an emerging body of literature~\cite{DeFarias:2003,Abbasi-Yadkori:2014,Chen:2018,Sutter:2017,Laksh:2018,MohajerinEsfahani:2018,Wang:2019,Lee:2019a,Banjac:2019,Beuchaut:2020,Martinelli:2020,Cheng:2020,Jin:2020,Shariff:2020,Bas-Serrano:2020a,Bas-Serrano:2020b} that studies ALP for the forward RL. In this paper we applied ALP techniques for the design of a provably efficient IRL algorithm. We hope that our techniques will be useful for future algorithm designers and will lay foundations for more exhaustive research in this direction. We discuss a few interesting directions.

     \begin{itemize}
    \item \emph{Extension to continuous MDPs.} In this case the primal-dual LPs are infinite-dimensional. Indeed, the occupancy measures $\mv$ live in a measure space, the $\uv$ variables live in a function space, and the constraints are uncountably infinite. Also the vector norms are replaced by the total variation norm of measures and the sup-norm of functions. We have to impose appropriate regularity assumptions on the MDP model, so that the linear operators $\op$ and $\op^*$ are well-defined and continuous. Moreover, we have to show that strong duality holds, since this is not anymore straightforward. We expect, that these technical issues can be tackled by similar tools as for forward MDPs (Ch.6 in~\cite{Hernandez-Lerma:1996}). A more crucial point is that while one can adopt our approach, there is no analytical expression for the extracted policy $\pi_{\thN}$ as for finite MDPs. A remedy for this difficulty could be to add a sparse solution promoting regularizer for the $\mv$-variable as proposed e.g., in~\cite{Boyd:2017}. 
    
    \item \emph{Extension to more general and other IRL settings.} Our methodology is naturally extensible to the more general case of RL with convex cosnstraints~\cite{Miryoosefi:2019}. A less straightforward extension is IRL with risk-sensitive objectives~\cite{Brown:2020a}. We consider this as an exciting future direction.
     	\item \emph{Imitation learning as $f$-divergence minimization.} In this work we utilize Lagrangian duality and propose a primal-dual algorithm for the resulting bilinear minimax formulation of the LfD problem. On the other hand, \cite{Nachum:2019a,Nachum:2020} regularize the objective of the LP formulation for the policy evaluation problem with an $f$-divergence, so that the  Fenchel-Rockafellar dual is an unconstrained convex program over $Q$-functions. In this way, one avoids possible complications due to the nested min-max optimization in the Lagrangian (e.g., in our case the need of strong value function features), but is not able to use standard stochastic optimization techniques because the gradient estimates become biased. Moreover, one can derive an actor-critic-type method~\cite{Nachum:2019b}, in which a $Q$-value is learned to minimize the \emph{squared Bellman error} and a policy $\pi$ is learned to maximize the $Q$-values. The same reasoning can be applied for imitation learning with $f$-divergences~\cite{Ghasemipour:2020,Kostrikov:2020}, which is a generalization of the maximum causal entropy IRL framework in~\cite{Ho:2016b}. While these algorithms are theoretically justified up to a point, they come without convergence guarantees. Indeed, the squared Bellman error is not convex and can lead to instabilities. A possible remedy is to use the convex and smooth \emph{logistic Bellman error} introduced in~\cite{Bas-Serrano:2020a} which arises from a different regularized $Q$-LP formulation. This approach is quite intriguing, requires careful design and analysis, and will be the subject of future work. 
     \end{itemize}

	\end{document}